%% file: main.tex
\PassOptionsToPackage{dvipsnames,table}{xcolor}

\documentclass[pmlr]{jmlr}
\usepackage{amsmath,amssymb,graphicx,url, bbm}
\jmlrvolume{329}
\jmlryear{2026}
\jmlrworkshop{Conformal and Probabilistic Prediction with Applications}
\jmlrproceedings{PMLR}{Proceedings of Machine Learning Research}

\input{preamble}

\title[Particle-Based CP Uncertainty Calibration in Stratified Configuration Spaces]{Particle-Based Conformal Prediction for Contact-Aware Uncertainty Calibration in Stratified Configuration Spaces}

\author{
 \Name{Lu\'is Marques{$^\dagger$}} \Email{lmarques@umich.edu}\\
 \addr{Department of Robotics, University of Michigan, Ann Arbor, MI, 48109, USA}\\
 \Name{Kristian Popov{$^\dagger$}} \Email{kpopov@umich.edu}\\
 \addr{Department of Aerospace Engineering, University of Michigan, Ann Arbor, MI, 48109, USA}\\
 \Name{Dmitry Berenson} \Email{dmitryb@umich.edu}\\
 \addr{Department of Robotics, University of Michigan, Ann Arbor, MI, 48109, USA}
}

\editor{Ernst Ahlberg, Ulf Johansson, Henrik Boström, Alberto Carlevaro, Johan Hallberg Szabadváry and Lars Carlsson}

\begin{document}

\maketitle
\renewcommand{\thefootnote}{\fnsymbol{footnote}}
\footnotetext[2]{Denotes equal contribution.}
\renewcommand{\thefootnote}{\arabic{footnote}}
\begin{abstract}
\looseness-1 Reliable uncertainty representation is essential for deploying autonomous systems that interact with their environment, as robots must reason about how uncertainty arising from both stochasticity and model mismatch is impacted by contacts with obstacles (e.g., when navigating through a cluttered environment or inserting a part into an assembly).
We propose \methodNameLong (\textbf{\method}), a geometry-aware, conformal prediction-based algorithm that generates probabilistically valid prediction regions of the unknown future system configuration using particle-based models of arbitrary fidelity. 
While calibrated uncertainty predictions are essential for safe and efficient planning, analytical or learned motion models are often inaccurate---due to limited data, simplifying assumptions, unmodeled effects, etc.---which can lead to unsafe executions or task failure.
Additionally, when a robot contacts an obstacle, the distribution of its future configurations can become multimodal or disjoint, or lie along manifolds of lower intrinsic dimension than the space of possible robot configurations.
Our method uses a calibration dataset of system transitions to locally calibrate motion uncertainty estimates, constructing regions guaranteed to contain the future robot configuration at a user-set probability. Our calibration procedure captures how motion uncertainty varies between contact-rich and contactless motions, leading to sufficient coverage in both cases.
We evaluate our method on two simulated planning tasks: controlling a marble around a labyrinth and performing tight-tolerance peg-in-hole insertion with a manipulator. Compared to relevant baselines, \method achieves the user-specified coverage requirement both in and out of contact and achieves up to a $30\%$ absolute improvement in task success rate over the best baseline.
Project website: \projWebsite\par 
\end{abstract}
\begin{keywords}
Uncertainty quantification, conformal prediction, stratified configuration spaces, contact-aware planning, uncertainty-aware planning. 
\end{keywords}

\section{Introduction}
\vspace{-2mm}
\looseness-1 Deploying robots in real-world contact-rich environments---e.g., maneuvering in cluttered homes, low-clearance key/peg insertion, and manipulation in clutter---requires robots that can reason about how state and action uncertainty can significantly change across interaction regimes. 
For example, uncertainty about a manipulator's end-effector location can be drastically reduced by moving from a free-space configuration to sliding contact. Likewise, impacts, sticking contact, or transitions between contact modes can produce abrupt changes in the distribution of future robot configurations.
\textit{Aleatoric uncertainty}, which cannot be reduced, arises naturally in real systems due to measurement error, actuation delay, surface roughness, small interaction-caused material deformations, and so on. Additionally, \textit{epistemic uncertainty} can arise due to structural model mismatch, wear and tear, or limited task-relevant data.
Epistemic uncertainty is present in both classical approaches, where simplifying modeling assumptions often contradict physical reality, and in learned approaches, where models are naturally unreliable when out-of-distribution (OOD).
While prior work has demonstrated that intentional contact and compliance are important for uncertainty reduction in precise manipulation settings \citep{800,406944,rodriguez2021unstable}, it is not always trivial to induce contact safely when under uncertainty. 
In many contact-rich tasks, violating force or positional constraints can lead to robot, environmental, or human damage. 
Yet, contact interactions can induce multimodal, discontinuous, and \textit{\transdimensional}\footnote{\looseness=-1 A \emph{\transdimensional} distribution assigns positive probability to configuration space regions of different intrinsic dimensions, e.g., both full-dimensional free space and lower-dimensional contact manifolds \citep{calder2017planning}.} uncertainty distributions over future robot configurations, which are challenging to model.
There is thus a need for contact-aware provably calibrated uncertainty estimates. We seek to provide finite-sample guarantees that are valid across multiple interaction modes and demonstrate their utility for safe and efficient motion planning under stochasticity and significant model mismatch.
\newline\indent
Conformal prediction (CP) \citep{papadopoulos2002inductive,vovk2022algorithmic} is a distribution-free framework that, given access to a calibration dataset of input-output pairs and a predictive model, enables the construction of prediction regions with probabilistic coverage guarantees. 
CP's sets provably contain the true unknown output at a user-specified likelihood, without placing strong assumptions on the structure or accuracy of the predictive model. CP applied to robotics has often provided a single global uncertainty bound on next-configuration predictions. Yet, system uncertainty can vary drastically across the robot's action and configuration space and is also affected by the nearby environment's geometry. Contact-rich planning requires methods that can provide contact-aware guarantees beyond marginal validity, along with contact-aware uncertainty bounds.
\newline\indent
To tackle these challenges, we propose \methodNameLong (\textbf{\method}), a method for constructing next-configuration prediction regions with provable probabilistic guarantees in contact-rich settings.
Our method uses an approximate particle-based dynamics predictor of arbitrary fidelity and explicitly considers the geometric structure of the robot's configuration space (C-space) to generate valid prediction regions that can represent \transdimensional uncertainty without placing probability mass in infeasible C-space regions. Our main contributions are:
\begin{enumerate}[
    label=\roman*),
    leftmargin=*,
    itemsep=-1pt,
    topsep=1pt
]
    \item \looseness=-1 An algorithm for constructing contact-aware, one-step \transdimensional prediction regions guaranteed to contain the unknown future robot configuration with at least a user-specified probability, despite epistemic and aleatoric uncertainty.
    \item \looseness=-1 A proof and numerical validation of our stratum-aware group-conditional guarantees using an approximate dynamics model of arbitrary fidelity and a finite calibration set.
    \nopagebreak[4]
    \item \clubpenalty=0 \widowpenalty=0 \interlinepenalty=0 Simulation experiments of a marble labyrinth control task and a manipulator peg-insertion task, demonstrating the utility of our approach for contact-rich motion planning under significant uncertainty.\footnote{See \projWebsite for algorithm-comparison videos on both tasks.}
\end{enumerate}

\section{Related Work}

\looseness-1 \textbf{Mondrian conformal prediction methods} partition examples according to the input space, output space (label-conditional), or, more generally, both, and perform calibration independently in each partition. The resulting per-partition conformal thresholds can capture heterogeneity in prediction uncertainty and yield tighter prediction regions while retaining finite-sample coverage guarantees \citep{vovk2003mondrian,lofstrom2015bias,pmlr-v128-bostrom20a,cabezas2024regressiontreesfastadaptive}. In parallel, probabilistic conformal prediction (PCP) \citep{wang_probabilistic_2022} enables the construction of prediction regions using particle-based predictive models, thereby enabling calibration of nonparametric distributional representations. Our method builds on both Mondrian CP and PCP to generate dynamical particle-based prediction regions that vary with the robot's configuration, action, and nearby geometry.
\newline\indent
\textbf{Planning under uncertainty} is often tackled by representing predictive uncertainty with symmetric parametric distributions such as Gaussians \citep{pets2018,schoellig19ECC}. 
However, in environments and tasks where contactful interactions are helpful, Gaussian beliefs might not accurately represent the post-contact distribution of the robot's future configuration, which may be \transdimensional.
To resolve this, other approaches represent uncertainty using particles sampled from a dynamics model or simulator \citep{calder2017planning,wirnshofer2018robust}.
Particle-based representations naturally capture multi-modality and allow uncertainty to collapse onto manifolds of varying intrinsic dimension, making them well-suited to contact-rich dynamics.
\newline\indent
\looseness-1 \textbf{Conformal prediction in robotics} has enabled the creation of safety filters \citep{Strawn_2023}, warning systems \citep{luo_sample_2024}, calibrated state estimators \citep{yang2023safe}, Lie-algebraic uncertainty estimators \citep{claps}, and the construction of safe trajectories through dynamic environments \citep{lindemann2023safe,lindemann2023adaptive,huang2025inter}. 
More closely related to our work, \citet{marques_quantifying_2025,wafr26} proposed a local CP method to calibrate linear-Gaussian dynamics models, yielding state-action-dependent next-state prediction sets.
However, these works rely on parametric (Gaussian) uncertainty predictions, do not consider how uncertainty might vary across contact regimes, and construct prediction regions that can cover infeasible regions in C-space. Instead, we use particle-based predictions to build feasible, \transdimensional uncertainty sets.
\newline\indent
\textbf{Planning in stratified C-spaces} has been tackled by both explicitly representing robot C-space as a stratified space \citep{goodwine_motion_2002,wei_stratified_2004,harmati_fitted_2002} and implicitly representing it as multiple manifolds of differing intrinsic dimension \citep{berenson2009manipulation,englert2020sampling}.
These approaches use the geometric structure induced by contact to improve planning robustness in contact-rich tasks, but generally assume access to a sufficiently accurate dynamics model for planning.
Our work uses the stratified structure of the configuration space to construct probabilistically valid prediction regions capable of representing contact-dependent, \transdimensional uncertainty. These regions can then inform a model predictive control (MPC) planner that uses an approximate model subject to both aleatoric and epistemic uncertainty.

\vspace{-4mm}
\section{Problem Statement}\label{sec:prob-statement}
\vspace{-2mm}
\textbf{Configuration space (C-space).} We formulate our planning problem in C-space, the space of the robot's degrees of freedom (DoF), rather than directly in the physical 2D or 3D workspace \citep{cspace_lp}.
For example, the C-space of a 7-DoF manipulator with seven revolute joints is a 7-dimensional torus\footnote{This is the case when the manipulator's joint limits are not considered.} where each dimension represents the angle of each joint. 
Representing planning problems in C-space can be advantageous, as the robot's configuration becomes a point and continuous robot motion becomes a path.
Workspace obstacles induce infeasible regions in C-space, where robot geometry intersects the environment. Classically, sampling-based motion planners have searched for collision-free trajectories through the feasible C-space
\citep{rrt,prm}.
In contact-rich robotic tasks, however, motion along obstacles remains feasible but constrains the allowable motion of the system, causing the dynamics to evolve on lower-dimensional manifolds within the C-space.
For example, a robot manipulating an object may transition from moving the object in free space, to sliding it on a table, to jamming it in a corner.
These contact-induced constraints produce a stratified configuration space composed of manifolds with differing intrinsic dimension.
\newline\indent
\looseness-1 \textbf{Dynamics and objective.} Consider a discrete-time Markovian stochastic system with full configuration $\superConfig \in \superCspace$ evolving according to unknown dynamics $\trueDynamics$, such that $
    \nextState \sim \trueDynamics(\state, \action)$,
where $s_t =(\superConfig, \superVelocity) \in T\superCspace$ denotes the robot's state and $\action \in \actionSpace$ denotes the control input at time $t \in \mathbb N_0$.
We consider the problem of moving this system, safely and efficiently, from an initial state $\initState$ to a goal region $\goal$ in a known environment.
However, our ability to reliably estimate $\nextSuperConfig$ is practically hindered by uncertainty, because the true system dynamics $\trueDynamics$ are unknown and must be approximated from finite data. 
As a result, learned or analytical dynamics models $\approxDynamics$ are subject to aleatoric and epistemic uncertainty.
For example, $\approxDynamics$ might approximate a nonlinear $\trueDynamics$ as linear, provide low-fidelity estimates of frictional coefficients, or not model physical phenomena like stiction. 
Consequently, planning with an approximate model may lead to unsafe behavior or inefficient trajectories, especially in dynamic settings where uncertainties can compound over time.
Thus, in order to support reliable uncertainty-aware planning and control, we seek calibrated uncertainty estimates over the unknown future system configurations.
\newline\indent
Let $\Cspace$ denote the projection of the full configuration space $\superCspace$ over which we seek to construct calibrated prediction regions, and let $\config\in\Cspace$ denote the relevant components of the full configuration $\superConfig\in\superCspace$.
For example, in a marble labyrinth system $\superConfig$ might include both the marble position and the board's tilting angles, while $\config$ contains only the marble location in the plane.  We henceforth refer to $\Cspace$ as the C-space.
Let $\CPinput$ denote the information at prediction time, and $\CPoutput := \nextConfig$ the prediction target.
A transition resulting from commanding $\action$ from a known $\state$ on the unknown real system can be written as the input-output pair $(\CPinput, \CPoutput)\in \CPinputSpace \times \CPoutputSpace$.
To provide probabilistic coverage guarantees for $\trueDynamics$ without strong distributional or dynamics assumptions, we adopt the standard assumption in the CP literature: access to a finite \textit{calibration dataset} $(\calibset)$ of input-output pairs from the process, obtained under the same ``conditions'' that we will observe at test time.
Formally:

\begin{assumption}\label{ass:exchangeableDataset}
    We have a dataset of transitions that is exchangeable with test-time transitions; i.e., $\calibset :=\{(\CPinput_i, \CPoutput_i)\}_{i=1}^{\numcal}$ and $(\CPinputTest, \CPoutputTest)$ are exchangeable.\footnote{\looseness-1 A sequence of random variables is exchangeable if its joint probability distribution is invariant to permutations, implying no distributional shift between calibration and test time.}
\end{assumption}
\looseness=-1 We implicitly assume the current state to be known exactly.
Despite Assumption~\ref{ass:exchangeableDataset}, we allow $\pred$ to be pretrained/fit on an arbitrary dataset $\trainset$ that is not necessarily exchangeable with $\calibset$.
We aim to construct an input-output-dependent prediction region $\CPregion$ that contains the next unknown robot configuration $\CPoutputTest$ with at least probability $1-\failureRate$, i.e.,
\vspace{-2mm}
\begin{equation}\label{eq:margCoverage}
    \mathbb P \{ \CPoutputTest \in \CPregion(\CPinputTest) \} \ge (1-\failureRate),
    \vspace{-2mm}
\end{equation}
\looseness-1 where $\failureRate \in (0,1)$ is a user-defined acceptable failure rate.
Constructing a $\CPregion \subseteq \Cspace$ that is (marginally) \textit{guaranteed} to satisfy Equation~\eqref{eq:margCoverage} can enable uncertainty-aware, safety-critical tasks because the set $\CPregion$ can be used during planning to mitigate errors in the approximate dynamics model $\approxDynamics$. 
While this coverage requirement is necessary for safety, it is not sufficient for practical downstream use in planning. 
For example, predicting $\CPregion=\Cspace$ trivially satisfies Equation~\eqref{eq:margCoverage}, but is uninformative for planning since it does not vary with the robot's configuration, action, or local contact geometry and likely contains infeasible configurations.
Instead, we aim to make $\CPregion$ as \textit{tight} (volume-efficient) and \textit{adaptive}\footnote{Adaptive prediction sets are smaller in regions of low uncertainty and larger in regions of high uncertainty.} as possible while still satisfying the user-set $(1-\failureRate)$ coverage requirement.

 The true unknown distribution of future $\nextConfig$ is inherently constrained to feasible C-space regions $\feasibleCspace\subseteq\Cspace$, which in contact-rich settings include not only free-space configurations but also configurations along walls or at corners. 
That is, for a given $\CPinput$, the true unknown conditional distribution $p(\CPoutput \mid \CPinput)$ assigns probability mass only to configurations that are physically possible under the system's constraints. For example, in an environment with rigid, immovable obstacles, configurations in which the robot's workspace occupancy intersects an obstacle should have probability zero. 
Therefore, prediction regions that include infeasible configurations are inherently volume-inefficient.
Beyond being inefficient, prediction regions that include infeasible configurations can degrade the ability of uncertainty-aware planners to reason about safety.
Since planners often reason about the entire prediction region when evaluating plans, building prediction regions that contain physically impossible configurations may lead to overly conservative plans.
Consequently, we aim to construct probabilistically valid prediction regions while obeying the configuration constraints imposed by the environment. To achieve this,
we use the fact that in some robotic systems, $\feasibleCspace$ can be decomposed into multiple smooth \textit{stratified submanifolds} of differing intrinsic dimension.
By modeling $\feasibleCspace$ as a \textit{stratified space}, we can construct prediction regions across the strata and inherently respect the constraints of the true system.
We now formalize this structure and treat the robot configuration space as a stratified space in the remainder of this work.
\begin{definition}[\citealp{tran_sampling_2020}]\label{def:stratified_space}
Let $M$ be a finite index set over the strata. A finite stratified space $\feasibleCspace$ is a disjoint union $\feasibleCspace:=\bigsqcup_{m\in M}\strata$, where each stratum is denoted by $\strata$, and $\indexer:\feasibleCspace\to M$ maps each configuration to the index $m$ of the stratum containing it.
\end{definition} 
The conditional distribution of the true dynamics $p(\CPoutput \mid \CPinput)$ may concentrate probability mass
within a single stratum or across multiple strata for a given
state-action pair.
We assume full knowledge of the constraints of the system, meaning we are able to define an explicit stratum indexer $\indexer$ that maps configurations to their stratum indices. Formally:
\begin{assumption}
    \looseness=-1 The robot's feasible C-space $\feasibleCspace$ is a finite stratified space with a known stratum indexer $\indexer$. The support of $p(\CPoutput \mid \CPinput)$ may span one or more strata depending on $\CPinput$.
\end{assumption}
In our work, $\densityEstimate$ denotes a conditional predictive distribution of the future configuration.

\vspace{-4mm}
\section{Conformal Prediction Variants}
\vspace{-2mm}
We briefly review relevant CP concepts to situate our method in the literature; see \citet{vovk2022algorithmic,angelopoulos_theoretical_2024} for a more comprehensive overview.

\vspace{-3mm}
\subsection{Split Conformal Prediction (\splitcp)}\label{sec:splitCP}
\vspace{-2mm}

\textit{Split/inductive conformal prediction} (\splitcp) \citep{papadopoulos2002inductive,Lei2018distribution} uses a dataset of input-output pairs $(\calibset)$ that must be \textit{exchangeable} with a test case $(\CPinputTest, \CPoutputTest)$ to provide a provable probabilistic upper bound ($\CPthreshold$) on the prediction error of a fixed model $\pred$. 
Given a user-specified acceptable failure rate $\failureRate \in (0,1)$, \splitcp enables the construction of sets $(\CPregion)$ in the output space $(\CPoutputSpace)$ that contain the unknown test label $(\CPoutputTest)$ with probability at least $(1-\failureRate)$.
Note that the predictive model $\pred: \CPinputSpace \to \predspace$ produces an estimate in the prediction space $\predspace$, which is not necessarily the output space $\CPoutputSpace$.
Let $\scoreFunc : \predspace \times \CPoutputSpace \to \mathbb R$ be a user-designed \textit{symmetric} nonconformity score\footnote{
A nonconformity score is symmetric if the resulting $R_i$ are invariant to permutations of the calibration dataset. In this work, we consider elementwise scores of the form $\scoreFunc(\pred(\CPinput_i),\CPoutput_i)$, which are symmetric.} that assigns larger values to ``worse'' model predictions and lower values to more accurate predictions. For each calibration example $(\CPinput_i,\CPoutput_i)\in\calibset$, we can compute its score as $R_i = \scoreFunc(\pred(\CPinput_i),\CPoutput_i)$.
\splitcp then constructs the threshold $\CPthreshold$ as a finite-sample-corrected quantile of the calibration scores:
\vspace{-2mm}
\begin{equation}
    \CPthreshold := \text{Quantile}\left(\{R_i\}_{i=1}^{\numcal} \cup \{\infty \}; \frac{\lceil (\numcal + 1)(1-\failureRate)\rceil}{\numcal+1}\right).
    \vspace{-2mm}
\end{equation}
By exchangeability, the score of the test point $R_{n+1} :=
\scoreFunc(\pred(\CPinputTest),\CPoutputTest)$ satisfies
$\mathbb P\{R_{n+1}\le \CPthreshold \}\ge (1-\failureRate)$. We can then define the prediction region as
\vspace{-2mm}
\begin{equation}
    \CPregion(\CPinputTest):= \{y\in \CPoutputSpace : \scoreFunc(\pred(\CPinputTest), y) \le \CPthreshold \},
    \vspace{-2mm}
\end{equation}
which by construction satisfies the \textit{marginal coverage guarantee} of Equation~\eqref{eq:margCoverage}.
Despite the marginal guarantee holding for many possible $r$, the choice of $\scoreFunc$ greatly impacts $\CPregion$'s volume efficiency. 
Furthermore, \splitcp produces a single input-output-independent $\CPthreshold$ and thus does not, by itself, account for heterogeneity in prediction uncertainty across the input and output space, making $\CPregion$ not adaptive.
We now consider how to address these problems.

\vspace{-3mm}
\subsection{Probabilistic Conformal Prediction (PCP)}\label{sec:pcp}
\vspace{-2mm}

Probabilistic conformal prediction (PCP) \citep{wang_probabilistic_2022} considers settings where $\pred$ does not return a single point prediction, but rather particle samples from an approximate conditional distribution, $\densityEstimate$. 
Let $\predParticleCloud{i}=\{\predParticle{i}{1},\ldots,\predParticle{i}{L}\}$ be a set of $L$ independent samples drawn from our dynamics predictor, with $\predParticle{i}{l}\sim\hat p(Y\mid\CPinput_i)$.
We can define an augmented calibration dataset and an augmented test point as
\vspace{-2mm}
\begin{equation}
    \pcpAugset := \{(\CPinput_i, \CPoutput_i,\predParticleCloud{i}) \}^{\numcal}_{i=1},\quad \text{and} \quad (\CPinputTest, \CPoutputTest, \predParticleCloud{n+1}),
\end{equation}
\vspace{1mm}
\looseness-1 where $\CPoutputTest$ is unknown and the samples are generated from the inputs $\CPinput$. PCP then defines the score $\scoreFunc$ as the minimum norm between the true output $\CPoutput$ and the prediction particles $\predParticleCloud{i}$. \citet{wang_probabilistic_2022} show that this formulation still provides the desired marginal coverage:
\begin{restatable}[Theorem 1 of \citet{wang_probabilistic_2022}]{theorem}{pcpproof} \label{thm:pcp}
Given $\pcpAugset$ and $R_i = \min_{1\le l \le L} \lVert \CPoutput_i - \predParticle{i}{l}\rVert$, the set constructed from applying \splitcp satisfies marginal coverage, Equation~\eqref{eq:margCoverage}.
\end{restatable}
While PCP uses the minimum-norm distance, other nonconformity scores that are deterministic given $\pcpAugset$ and symmetric in the $L$ sampled particles would also maintain the marginal coverage guarantee \citep{kuchibhotla_exchangeability_2021}.
Particle-based dynamics predictors are widely used in robotics to represent contact-rich motion, which is often multimodal and discontinuous. This makes PCP particularly relevant to our task, as we look to construct probabilistically valid uncertainty sets without requiring an explicit parametric form for $\densityEstimate$.
However, PCP's conformal uncertainty upper bound $\CPthreshold$ is still global, despite our domain knowledge that prediction error should depend on a system's state, action, and configuration stratum.

\vspace{-3mm}
\subsection{Mondrian Conformal Prediction (MondrianCP)}\label{sec:mondrianCP}
\vspace{-2mm}

\looseness=-1 \textit{Mondrian conformal prediction} \citep{vovk2003mondrian} extends \splitcp by partitioning examples into $\numGroups$ disjoint groups and calibrating each group independently, i.e.,
instead of computing a single global threshold $\CPthreshold$ across all examples (\splitcp), MondrianCP computes a separate threshold per group $\CPthreshold_{\groupIndex}$.
This provides per-group coverage guarantees rather than marginal guarantees---where the partition may be defined over the feature space $\CPinputSpace$, the output space $\CPoutputSpace$, or the joint input-output space $\CPinputSpace \times \CPoutputSpace$.
Formally,
let $\CPinputSpace\times \CPoutputSpace$ be partitioned into $\numGroups$ disjoint sets, and $g: \CPinputSpace\times \CPoutputSpace \to \{1,\ldots, \numGroups \}$ be a map assigning each example to one of the $\numGroups$ groups.
Mondrian CP constructs prediction regions satisfying the user-specified \textit{group-conditional coverage}
\vspace{-1mm}
\begin{equation}\label{eq:group-coverage}
    \mathbb P\{\CPoutputTest \in \CPregion(\CPinputTest) \mid g(\CPinputTest, \CPoutputTest ) = \groupIndex \} \ge 1-\failureRate, \ \text{for all }\groupIndex \in \{ 1, \ldots, \numGroups\}.
    \vspace{-1mm}
\end{equation}
Let $\mathcal I_{\groupIndex}$ denote the indices of the calibration data points that are assigned to group $\groupIndex$, i.e., $\mathcal I_{\groupIndex} := \{ i \in [\numcal ]:g(X_i, Y_i)=\groupIndex\}$. 
$\lvert \mathcal I_{\groupIndex}\rvert$ denotes the number of $\calibset$ points in group $\groupIndex$. For each group, the conformal score quantile is
\vspace{-2mm}
\begin{equation}\label{eq:CPquantile-mondrian}
    \hat q_{\groupIndex} := \text{Quantile}\left(\{R_i\}_{i\in \mathcal I_{\groupIndex}} \cup \{\infty\}; \frac{\lceil (\lvert \mathcal I_{\groupIndex}\rvert+1)(1-\failureRate)\rceil}{\lvert \mathcal I_{\groupIndex}\rvert +1 } \right),
    \vspace{-2mm}
\end{equation}
yielding the MondrianCP prediction region
\vspace{-2mm}
\begin{equation}\label{eq:CPregion-mondrian}
    \CPregion(\CPinputTest) = \{y \in \CPoutputSpace: r(\pred(\CPinputTest),y) \le \hat q_{g(\CPinputTest, y)} \}.
    \vspace{-2mm}
\end{equation}
Under Assumption~\ref{ass:exchangeableDataset} and the dataset-independent group-assignment map $g$, MondrianCP preserves valid coverage within each group.
Intuitively, conditioning on a fixed group reduces to performing conformal calibration independently on the subset of $\calibset$ assigned to that group.
More formally, applying the Mondrian coverage guarantee \citep{angelopoulos_theoretical_2024}:
\begin{restatable}{theorem}{mondrianproof} \label{thm:mondrian}
Given Assumption~\ref{ass:exchangeableDataset} and a symmetric $\scoreFunc$, the Mondrian prediction region in Equation~\eqref{eq:CPregion-mondrian} satisfies the group-conditional coverage guarantee in Equation~\eqref{eq:group-coverage} for every $\groupIndex\in \{1,\ldots, \numGroups \}$ with $\mathbb P\{g(\CPinputTest, \CPoutputTest)=\groupIndex\} > 0$.
\end{restatable}

\looseness-1 In our setting, the group-assignment map $g$ will be informed by the robot's predicted future C-space strata, as well as its current state and applied action. In contact-rich motion planning, prediction uncertainty can vary significantly across different strata of the robot's C-space.
For example, sliding contact along a surface can display substantially different dynamics and uncertainty characteristics than free-space motion.
Hence, a single globally calibrated threshold can become overly conservative in low-uncertainty regions (or overly optimistic in high-uncertainty regions), making MondrianCP well-suited for our setting.

\vspace{-5mm}
\section{Method: \method}\label{sec:method}
\vspace{-1mm}
\looseness-1 We first show how to build probabilistically valid prediction regions ($\CPregion$) for a stratified configuration space using an approximate dynamics model and a finite dataset of true system transitions. These prediction regions provide finite-sample coverage guarantees with respect to the unknown future system configurations and, by construction, contain only feasible configurations in C-space. We then use $\CPregion$ to construct uncertainty-aware motion plans that account for both aleatoric and epistemic uncertainty. Figure~\ref{fig:calibration} illustrates the offline group partitioning and calibration procedure, while Figure~\ref{fig:inference} illustrates the online prediction-region construction process used during planning. Algorithms~\ref{alg:SParCP} and~\ref{alg:rollout_SPar-CP} describe our specific implementation (\method), from calibration data processing to downstream planning use.

\begin{figure}[b!]
\vspace{-5mm}
    \centering
\includegraphics[width=\linewidth]{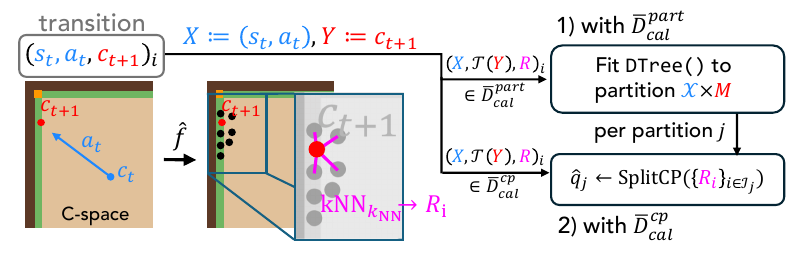}
    \vspace{-9mm}
    \caption{\looseness-1 \textbf{Offline calibration of \method.} For each transition in $\calibset$, $\approxDynamics$ receives the current state $\state:=(\superConfig,
    \superVelocity)$ and action $\action$ to sample $L$ predictive particles (black) of the future configuration $\nextConfig$, which we use to compute the nonconformity score $R_i$ (pink). The augmented subset $\augset^{part}$ is used to fit a regression decision tree that separates the input space---state, action, and future stratum index $\indexer(\nextConfig)$---into groups with approximately group-constant prediction scores. The holdout subset $\augset^{cp}$ is then passed through the \dtree, with each example landing in a leaf node and hence a corresponding group $\groupIndex$. SplitCP is performed independently for each group, resulting in a per-partition conformal threshold $\CPthreshold_{\groupIndex}$. 
    }
    \label{fig:calibration}
    \vspace{-4mm}
\end{figure}

\input{algo1}

\input{algo2.tex}

\subsection{Theoretical Analysis}\label{sec:theoretical-analysis}
We first show how to construct conformal prediction regions on a stratified C-space by defining Mondrian groups as a function of the current state-action pair and the future configuration stratum.
Consider a stratified feasible C-space $\feasibleCspace:=\bigsqcup_{m\in M}\strata,$ together with a known stratum indexer $\indexer:\feasibleCspace\to M$ mapping each configuration to the index of its corresponding stratum.
We introduce the grouping function
\vspace{-2mm}
\begin{equation}\label{eq:strat-grouping}
    g : \CPinputSpace \times M \to \{1, \dots, \numGroups\},
    \vspace{-2mm}
\end{equation}
which partitions examples according to the prediction-time information $\CPinput$ and the future C-space stratum $\indexer(\CPoutput)$.
This grouping allows the calibration to adapt across C-space strata with distinct dynamics and uncertainty structures.
To preserve conformal validity, the grouping function $g$ must be fixed independently of the calibration data used to compute the conformal thresholds.
For each induced group $\groupIndex \in \{1,\dots,\numGroups\}$, we follow the Mondrian procedure and define its calibration index set as
$
\mathcal I_{\groupIndex} := \{ i \in [\numcal] : g(\CPinput_i, \indexer (\CPoutput_i)) = \groupIndex \}
$.
\looseness-2 We then compute the conformal threshold within $\mathcal I_{\groupIndex}$ as in Equation~\eqref{eq:CPquantile-mondrian}.
To establish formal coverage guarantees, we use the following conditional exchangeability property \citep{angelopoulos_theoretical_2024}:
\begin{lemma}[Exchangeability property] \label{lem:condex}
\looseness=-1 Assume $\{(X_i,Y_i)\}_{i=1}^{n+1}$ are exchangeable. For any $\groupIndex$ with $\mathbb P\{g(X_{n+1},\indexer(Y_{n+1}))\allowbreak=\groupIndex\}>0$, conditional on $\mathcal I_{\groupIndex}$ and the event $g(X_{n+1},\indexer(Y_{n+1}))\allowbreak=\groupIndex$, the examples $\bigl((X_i,Y_i)\bigr)_{i\in\mathcal I_{\groupIndex}\cup\{n+1\}}$ are exchangeable.
\end{lemma}
We are now able to provide the following coverage guarantee.
\begin{theorem} \label{thm:stratcp}
    Given Assumption~\ref{ass:exchangeableDataset}, the grouping function $g$ defined in Equation~\eqref{eq:strat-grouping}, and a symmetric score function $\scoreFunc$, let the \method prediction region be
    \begin{equation*}
        \CPregion(\CPinputTest)
        := \{y\in\feasibleCspace:
        \scoreFunc(\pred(\CPinputTest),y)
        \le \hat q_{g(\CPinputTest,\indexer(y))}\}.
    \end{equation*}
    For $\failureRate \in (0,1)$ and every group $\groupIndex$ with $\mathbb P\{g(\CPinputTest,\indexer(\CPoutputTest))=\groupIndex\}>0$, the \method prediction region $\CPregion(\CPinputTest)$ satisfies the group-conditional coverage guarantee
    \begin{equation*}
        \mathbb P \{\CPoutputTest \in \CPregion (\CPinputTest) \mid g(\CPinputTest, \indexer (\CPoutputTest) ) = \groupIndex\} \ge 1-\failureRate.
    \end{equation*}
\end{theorem}
\begin{proof}
\looseness=-1
By Lemma~\ref{lem:condex}, conditional on $\mathcal I_{\groupIndex}$ and $g(X_{n+1},\indexer(Y_{n+1}))=\groupIndex$, exchangeability of the examples and symmetry of $\scoreFunc$ imply exchangeability of the scores within group $\groupIndex$.
The standard SplitCP argument therefore gives coverage under this conditioning.
Averaging over $\mathcal I_{\groupIndex}$ yields the stated guarantee conditional on $g(X_{n+1},\indexer(Y_{n+1}))=\groupIndex$.
\end{proof}

Theorem~\ref{thm:stratcp} establishes coverage conditional on the group $g(\CPinput, \indexer (\CPoutput)) = \groupIndex$. This group index is a function of the stratum of the unknown test output $\CPoutput$, which is not available at inference time.
Thus,
to obtain inference-time prediction regions, we evaluate the group assignment $g$ for all candidate output stratum indices $\indexer (\CPoutput)$ and construct the final prediction region $\CPregion$ as the union of the corresponding group-specific regions (see Figure~\ref{fig:inference}).
Because the strata are disjoint and the union of all strata completely covers the feasible output space, the group-conditional guarantees combine to yield marginal coverage guarantees by the law of total probability, i.e., the union-over-strata construction also provides marginal coverage.
\begin{restatable}[Marginal Coverage by Union of Output Strata]{restcorollary}{margcov}
\label{cor:margcov}
Let $\CPregion(\CPinputTest)$ be the union of the group-specific
prediction regions obtained by evaluating $g(\CPinputTest,m)$ for every
output stratum index $m\in M$.
Then $\CPregion$ satisfies the marginal coverage guarantee
\[
\mathbb P\{\CPoutputTest \in \CPregion(\CPinputTest)\} \ge 1-\failureRate.
\]
\end{restatable}
\vspace{-4mm}
\textbf{\proofname}\ See Appendix~\ref{app:proofs}.

Let $\obstacleCspace:=\Cspace\setminus\feasibleCspace$ denote the infeasible regions in C-space---those occupied by obstacles or beyond joint limits.
We define each stratum $\strata$ to lie entirely within the feasible configuration space, i.e., $\strata \cap \obstacleCspace = \emptyset$, $\forall m\in M$.
Therefore, by construction, our prediction regions satisfy
\vspace{-2mm}
\[
\CPregion(\CPinputTest) \cap \obstacleCspace = \emptyset.
\vspace{-2mm}
\] In contrast, methods that construct prediction regions using distances in the ambient C-space, without accounting for its stratified structure, may extend into infeasible regions, yielding regions that are physically invalid and overly conservative.
We have thus shown how stratified configuration 
spaces
can be incorporated into Mondrian conformal prediction to construct feasible uncertainty-aware prediction regions for contact-rich robotic tasks.

\subsection{Proposed Implementation}\label{sec:implementation}
We now show a possible implementation of the general approach described in Section~\ref{sec:theoretical-analysis} for building stratified prediction regions. 
In some robotic settings, the uncertainty of the true system can be disjoint and multimodal, e.g., in contact-rich settings where small changes in state or action can result in discrete switches of the active constraint.
In such systems, particle-based dynamics representations are often used for motion planning. 
Assume we have access to an approximate dynamics model $\pred$ that can generate samples from an approximate conditional distribution over the next configurations. 
To construct prediction regions, we adopt the PCP framework from \citet{wang_probabilistic_2022}, which enables us to construct a calibrated prediction set using samples drawn from $\approxDynamics$. 
We define the nonconformity score $\scoreFunc$ as the distance between the true next configuration and the $\nearestNeighborIndex$-th nearest neighbor among the $L$ particles sampled from $\approxDynamics(\CPinput _i)$.
We choose the $\nearestNeighborIndex$-th nearest-neighbor score because it provides a continuously valued proxy for the local particle density~\citep{knn}.
Formally, the score function is
\vspace{-2mm}
\begin{equation}\label{eq:knn}
    r(\pred(X_i), Y_i) := \mathrm{kNN}_{\nearestNeighborIndex} (Y_i, \predParticleCloud{i}),
    \vspace{-2mm}
\end{equation}
\looseness-1 where $Y_i$ is the true next configuration, $\predParticleCloud{i}$ is the set of $L$ propagated particles, and $\mathrm{kNN}_{\nearestNeighborIndex}$ denotes the distance to the $\nearestNeighborIndex$-th closest particle.
Because this score is symmetric, we can 
construct probabilistically valid prediction regions using this scheme (Theorem~\ref{thm:pcp}).
While coverage guarantees hold for arbitrary $\nearestNeighborIndex$ and, more generally, for other symmetric scores that depend on the predicted particles and the true next configuration, these design choices can significantly impact downstream task performance. 
With this, we can now define the augmented calibration dataset, similarly to PCP, as $\augset := \{(\CPinput_i, \indexer(\CPoutput_i),R_i): (\CPinput_i,\CPoutput_i)\in \calibset\}$.

\begin{figure}[t]
    \centering
    \includegraphics[width=0.99\linewidth]{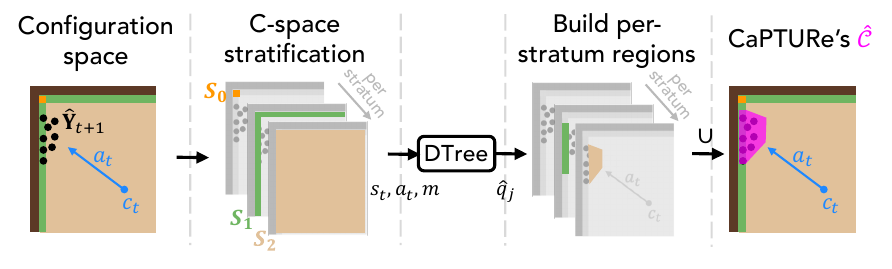}
    \vspace{-3mm}
    \caption{\textbf{Construction of stratified prediction region $\CPregion$}. Given an action $\action$ and current state $\state:=(\superConfig,\superVelocity)$, the predictive model $\approxDynamics$ returns predictive particles (black) representing possible future configurations. For each stratum index $m\in M$, we query the pre-fit \dtree with $(\state,\action,m)$ to get a contact-aware uncertainty threshold $\CPthreshold_{\groupIndex}$. We then evaluate which configurations in the corresponding stratum $S_m$ have score $\le \CPthreshold_{\groupIndex}$, where the score is the $\nearestNeighborIndex$-th distance between a candidate configuration in $S_m$ and the particles. Finally, we compose $\CPregion$ as the union of the per-stratum regions evaluated in the previous step.}
    \label{fig:inference}
\end{figure}

\looseness=-1 Similarly, although any fixed grouping function $g$ yields probabilistically valid prediction regions, the choice of $g$ can greatly affect the downstream efficiency and adaptivity of the conformal sets.
For example, choosing a grouping function that depends exclusively on the future configuration stratum creates one stratum-specific threshold $\hat q_m$, which could ignore how prediction uncertainty varies across different states and actions and possibly lead to undercoverage in some regimes. 
To capture local predictive uncertainty variations in a data-driven manner, we define our $g$ using the LOCART procedure, first presented by \citet{cabezas2024regressiontreesfastadaptive} and later adapted for dynamical systems by \citet{marques_quantifying_2025,wafr26}.
Following LOCART, we randomly split the augmented calibration dataset into two disjoint sets $\augset:= \augset^{part} \sqcup \augset^{cp}$. 
First, we fit a regression decision tree (\dtree) with the \textit{augmented inputs} $(\CPinput_i,\indexer(\CPoutput_i))$, where $\CPinput_i=(\state,\action)_i$ in our implementation, and set $R_i$ as the target variable.
The tree partitions the joint $\CPinputSpace \times M$ space to find regions of approximately constant model error.
The leaves of the fitted \dtree define the grouping function $g(X,\indexer(Y))$ used by our algorithm.
Within each leaf (group), we compute a separate conformal threshold $\hat q_{\groupIndex}$ using only calibration points of $\augset^{cp}$ that were assigned to that group.
This process preserves the group-conditional coverage guarantees of Theorem~\ref{thm:stratcp}, while producing prediction regions that adapt to the dynamics' local uncertainty structure. 
Importantly, $g$ is learned independently of the calibration set $\augset^{cp}$ and fixed before inference, preserving exchangeability.

We emphasize that the coverage guarantees from Theorem~\ref{thm:stratcp} apply only to single-step prediction.
Providing formal multistep coverage guarantees is not trivial because $\CPregion$ at later planning steps depends recursively on predicted states, particles, and $\CPregion$ earlier in the horizon. 
Consequently, the exchangeability assumption no longer holds across rollout steps.
While prior work has studied multistep CP under simpler non-input-dependent settings, extending coverage guarantees to our state-action- and stratum-dependent construction in a data-efficient manner remains an important direction for future work. 
Still, our method can be used for model predictive control (MPC) over a fixed horizon $H$, as demonstrated in our experiments. In Section~\ref{sec:exp}, we recursively apply the proposed single-step construction along the MPC rollout horizon.
Although this heuristic application does not provide formal multistep guarantees, we find it effective in practice.

\section{Experimental Results}\label{sec:exp}
\vspace{-2mm}
\looseness-1 To validate $\method$'s local and stratum-aware one-step coverage guarantees and demonstrate its usefulness for contact-rich probabilistic motion planning with an approximate dynamics model, we evaluate \method on two contact-rich tasks subject to aleatoric disturbances and significant model mismatch.
\looseness=-1 In Section~\ref{sec:expMarble}, we control a marble around a tight-clearance maze environment while avoiding known pit locations, showing our method's ability to perform highly dynamic tasks where small dynamics prediction inaccuracies can lead to failure. In Section~\ref{sec:expPeg}, we control a robot manipulator to insert a round peg into a low-tolerance hole fixture, a common assembly task in manufacturing that is sensitive to sensor and dynamics errors. \newline\indent
We compare our method with four relevant baselines: 1) an uncalibrated particle-based planner using $\approxDynamics$ directly (\emph{\particleNoCP}); 2) \emph{\lucca} \citep{marques_quantifying_2025}, a local CP method that calibrates a Gaussian $\approxDynamics$ differently based on the system's state and action, resulting in adaptive hyperellipsoidal prediction regions; 3) \emph{\pcp} (Section~\ref{sec:pcp}), a particle-based conformal prediction approach that builds input-output-independent balls around each predicted particle; and 4) an ablation of our method (Section~\ref{sec:implementation}) that does not consider the next-configuration stratum index as part of the partitioning procedure (\emph{\ablationNoStrata}), i.e., the inputs to the \dtree become $\CPinput=(\state,\action)\in \mathcal S \times \mathcal U = \CPinputSpace$. We evaluate \method with $\nearestNeighborIndex=L/2$\footnote{This gives $\nearestNeighborIndex=8$ for the marble task ($L=16$ particles) and $\nearestNeighborIndex=4$ for the peg-insertion task ($L=8$).} in our experiments and also report results for $\nearestNeighborIndex=1$ as a further hyperparameter ablation. The $\nearestNeighborIndex=1$ variant produces prediction regions that are more sensitive to individual predicted particles, whereas $\nearestNeighborIndex=L/2$ reduces this sensitivity and empirically produces smoother $\CPregion(\CPinputTest)$. All methods share the same approximate model $\approxDynamics$ and calibration dataset $\calibset$.
For methods that perform local calibration (\lucca, \ablationNoStrata, \method), we randomly select $70\%$ of $\calibset$ for fitting the \dtree and use the remainder to obtain the per-group conformal thresholds $\CPthreshold_{\groupIndex}$, as detailed in Section~\ref{sec:method}.
While the particle-based methods sample from a predictive Gaussian model, \lucca calibrates said $\approxDynamics$ directly. We use $\failureRate =0.1$ for all experiments.
See \projWebsite for videos comparing rollouts produced by \method and the baselines across both tasks.

\subsection{Marble Labyrinth Control}\label{sec:expMarble}
\vspace{-2mm}
We simulated a planar marble control environment (leftmost view in Figure~\ref{fig:marble_visualize}), inspired by the popular BRIO labyrinth toy. 
\begin{figure}[!b]
    \centering
    \vspace{-5mm}\includegraphics[width=\linewidth]{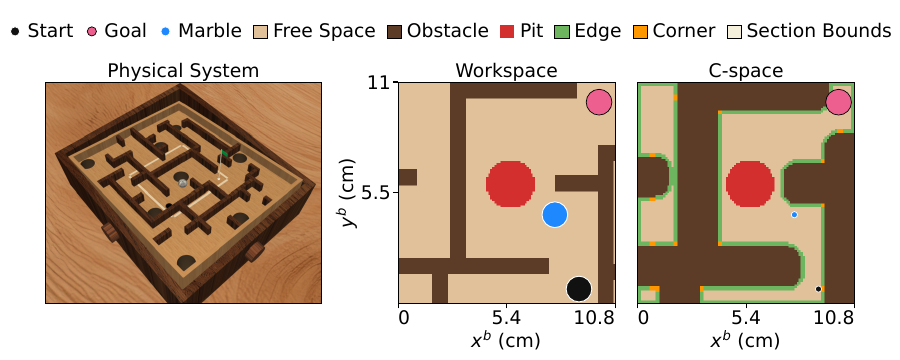}
    \vspace{-10mm}
    \caption{From left to right, in increasing levels of abstraction: 1) MuJoCo rendering of the full marble labyrinth with its tilting mechanism; the ``Center'' subsection is delineated by light borders. 2) Top-down view of the marble's workspace, which lies in the board plane. 3) Marble C-space used for motion planning. The walls are inflated by the marble radius, reducing the marble to a point robot. The Free Space (2D), Edge (1D), and Corner (0D) strata are colored beige, green, and orange, respectively.}
    \label{fig:marble_visualize}
\end{figure}
Following \citet{marbleDynamics}, the system's state is given by $\state=[x^b,\dot x^b, y^b, \dot y^b, \alpha, \beta]^\top$, where $(\cdot)^b$ indicates coordinates in the board-fixed frame and $(\alpha,\beta)$ are the plate inclination angles.
The control inputs are $\action =[\omega_1,\omega_2]\transpose$, motor rotational velocities that are related to the board's tilting rates $(\dot\alpha, \dot\beta)$ by a simplified linear relationship $\dot\alpha =\tiltGainAlpha \omega_1,\dot \beta =\tiltGainBeta\omega_2$.
The marble has a radius of 6.4 mm, and the pits have a radius of 12 mm. We assume that all walls and pits are known and define a fall as the marble's centroid entering a pit.
For both planning and physics stepping, we inflate the walls by the robot radius, moving from workspace (middle of Figure~\ref{fig:marble_visualize}) to C-space (right of Figure~\ref{fig:marble_visualize}). In C-space, we can treat the robot as a point robot.
The approximate dynamics model follows the discretized model of Equation~\eqref{eq:approximateMarbleStochastic}, from which particle-based methods sample $16$ particles at every planning step. However, the true marble moves according to Equation~\eqref{eq:realMarbleStochastic}, where a multiplicative dynamics term introduces considerable model mismatch.
For both the true and approximate systems, an additive disturbance is applied to the commanded control action, introducing aleatoric uncertainty. We used $\dt=0.1$ s for both $\approxDynamics$ and $\trueDynamics$. 
To enable safe maze navigation under dynamics uncertainty, we aim to construct $\CPregion$ providing coverage guarantees over the 2D marble position $\config = [x^b, y^b]\transpose$.\newline\indent
This is a challenging control problem due to the tight clearances between pits and walls and the system's limited control authority.
Since control inputs can only indirectly influence the marble location through the plate inclination, the system exhibits delayed responses and can accumulate significant momentum when inaccurately modeled. As a result, it can be challenging to decelerate or change directions quickly enough to avoid nearby pits.
Additionally, the marble's uncertainty evolution is strongly shaped by its interactions with the maze's walls. While in free space the marble motion is unconstrained, wall contacts reduce positional uncertainty along the wall-normal direction. Near edges and corners, the true uncertainty over future marble configurations can become disjoint and span both full-dimensional free space and lower-dimensional contact manifolds. These \transdimensional distributions might not be adequately captured by unimodal parametric forms such as Gaussians. This motivates our use of a particle-based implicit distribution representation capable of capturing contact-dependent uncertainty propagation.
Since we treat all walls and corners equally, the task induces the three strata visualized in the C-space panel of Figure~\ref{fig:marble_visualize}. 
To speed up computation, we precomputed a grid in $(x^b,y^b)$ offline, where each element has an associated stratum. Then, at test time, our stratum indexer $\indexer(\nextConfig)$ is a lookup on this position grid (see Appendix~\ref{app:marbledyn} for details). 
\looseness=-1 While relatively low-dimensional, this task serves as a representative testbed for broader planar dynamic manipulation problems. 
\newline\indent
\looseness-1 \textbf{Numerical coverage validation.} Let \textsc{lin}$(a,b,N)$ denote a linearly spaced sequence of $N$ real numbers between $a$ and $b$. A calibration dataset was collected for each maze section (see Figure~\ref{fig:marbleLabyrinth} for the location of each section within the full maze) by enumerating a Cartesian grid with states given by $(x^b,y^b)\in \textsc{lin}(\min_x,\max_x,8)\times \textsc{lin}(\min_y,\max_y,8)$, $(\dot x^b,\dot y^b)\in \textsc{lin}(-0.3,0.3,3)^2$, and $(\alpha,\beta)\in \textsc{lin}(-0.175,0.175,3)^2$, actions $(\omega_1,\omega_2)\in \textsc{lin}(-0.5,0.5,4)^2$, and performing one-step rollouts of $\trueDynamics$.
Positions initially in collision were projected outward, resulting in sufficiently many transitions in each of the three strata and $\lvert \calibset \rvert$ of approximately $147{,}000$ for Center, $127{,}000$ for Center Right, $137{,}000$ for Bottom Left, $165{,}000$ for Bottom Right, $124{,}000$ for Top Left, and $147{,}000$ for Top Center. To validate our per-group coverage guarantees empirically, we evaluate thousands of real-system transitions over each of the six maze subsections, again following a Cartesian grid as in the $\calibset$ collection (see Appendix~\ref{app:marbledyn} for details).
For each test transition, we propagate $10{,}000$ Monte Carlo (MC) samples using $\trueDynamics$ and build a prediction region $\CPregion$ for each method. 
For every MC particle, we determine the stratum of its resulting next configuration using the precomputed lookup indexer $\indexer(\nextConfig)$ and then evaluate whether the resulting configuration $\nextConfig$ is contained in $\CPregion$.
\looseness=-1 After evaluating all test transitions, we estimate the likelihood that particles landing in stratum $m$ are contained in $\CPregion$. This provides an empirical estimate of coverage conditional on the realized next-configuration stratum. To compare the sizes of prediction regions across methods, we voxelize the C-space and approximate the volume of each $\CPregion(\CPinputTest)$ as the total volume of the voxels in $\CPregion(\CPinputTest)$. We report the full results per map in Table~\ref{tab:marblecoveragebymap}, while in Table~\ref{tab:marblecoverage} we report averages across all $427,680$ test cases from the six maps. \newline\indent
\input{marble_testcases_table}
\looseness=-1 We do not analyze \particleNoCP because this baseline outputs a set of particles rather than a prediction region. \lucca, \pcp, and \ablationNoStrata achieve marginal coverage, as expected, but significantly undercover in free space while being overconservative in lower-dimensional strata.
This could be due to model mismatch leading to larger prediction uncertainty in unconstrained motion, as contacts can collapse the true next-configuration distribution along the contact normal.
Particle-based methods construct $\CPregion$ with significantly lower C-space volume than \lucca, which returns hyperellipsoidal prediction regions.
While PCP produces smaller prediction regions than our method, it does not satisfy the user-provided coverage requirement in all strata: it is significantly overoptimistic in Free Space and overconservative on Edges and Corners.
Both variants of our method achieve the $1-\failureRate=0.9$ coverage requirement per stratum, with the $\nearestNeighborIndex=8$ variant being slightly more volume-efficient. These results suggest that our stratum-aware approach can provide sufficient coverage across categorically different contact modes.
\newline\indent
\looseness-1 \textbf{Probabilistic motion planning.} We investigate whether the proposed $\CPregion$ improves downstream motion planning.
Assuming $\state$ to be known, we aim to safely navigate the marble around known pits toward a goal configuration. We use Model Predictive Path Integral Control (MPPI) \citep{williams2017information} as our MPC trajectory optimizer, due to its demonstrated ability to efficiently tackle problems with nontrivial dynamics and costs (see Appendix~\ref{app:mppiParams} for MPPI hyperparameters). 
After computing a plan, we execute its first action and replan recursively until either the task is complete or the marble has fallen into a pit.
The objective function, Equation~\eqref{eq:mppiCOstMarble}, balances task progress and safety by reducing the mean traversable distance between the propagated particles and \goal and penalizing the intersection area between the prediction region and the pits.
Figure~\ref{fig:marble_planning} displays the trajectories obtained in each of the six tested sections of the full maze, and Table~\ref{tab:marbleplanning} reports the per-section time to goal and success rate. Rollout videos are shown on the project website. \newline\indent
\input{marble_planning_table}
\looseness=-1 \particleNoCP can perform adequately on sections with greater clearance (e.g., Top Left). However, when safe maneuvering requires sharp calibrated uncertainty bounds, it can fall due to accumulating too much momentum. \pcp's performance is comparable. \lucca, due to building hyperellipsoids that do not respect local geometry, is generally overconservative. \ablationNoStrata provides mixed results, not being noticeably better than the uncalibrated baseline. Both this ablation and \lucca cannot distinguish between transitions landing in different configuration strata. Our method shows improved task success rate across the tested maps, with the $\nearestNeighborIndex=L/2=8$ variant performing best overall. This indicates that the higher sensitivity to individual predicted particles of the $\nearestNeighborIndex=1$ ablation can degrade planning performance despite still providing adequate coverage guarantees. Empirically, the smoother prediction regions produced by $\nearestNeighborIndex=8$ were easier to optimize over with the tested MPPI parameters.
These results suggest that \method can be helpful for complex contact-rich planar control tasks under significant uncertainty.
\vspace{-3mm}
\subsection{Tight-Tolerance Peg Insertion}\label{sec:expPeg}
\vspace{-2mm}
We further evaluate \method on a tight-tolerance peg insertion task adapted from the Factory simulation suite in Isaac Sim \citep{factory}. 
We control a 7-DoF manipulator (Franka Panda) to insert a cylindrical peg (diameter of $7.986 \text{ mm}$) into a $9.000 \text{ mm}$-diameter hole under both stochastic disturbances and significant model mismatch.
To facilitate contact-aware planning, we restrict the end-effector motion to the hole's plane, reducing its possible poses from the space of three-dimensional rigid-body poses, $SE(3)$, to the space of planar rigid-body poses, $SE(2)$. Hence, the peg configuration $\config\in SE(2)$ is parameterized by $[x,z,\theta]^\top$, where $(x,z)$ denotes the planar position of the peg's centroid and $\theta$ its planar orientation (see Figure~\ref{fig:cspaceviewPeg} for how the feasible C-space changes for different peg orientations).
The actions are desired $SE(2)$ displacements in the plane. Given the control input $\action$, the applied action $\boundedActionAt{t}$ is drawn from
\vspace{-2mm}
\begin{equation} \label{eq:true}
\boundedActionAt{t} = 1.5 \cdot \text{clip}(\action + \noise, \actionMin, \actionMax), \quad \text{with } \noise \sim \mathcal{N}(0, \text{diag}(1.5^2, 1.5^2, 1.5^2)),
\vspace{-2mm}
\end{equation}
where $\noise$ is a known additive white noise Gaussian disturbance introducing aleatoric uncertainty. The unknown multiplicative term $(1.5)$ introduces epistemic uncertainty.
Renderings of the simulation environment, the workspace and the C-space are shown in Figure~\ref{fig:factory_pih}.

Due to the small peg-hole clearance, successful insertion requires precise alignment.
However, near contact, even small configuration uncertainty can induce different physical interactions with the environment---the distribution over future peg configurations can become discontinuous and multimodal across contact regimes and can span both full-dimensional free space and lower-dimensional contact manifolds, producing \transdimensional uncertainty.
Additionally, overly optimistic future configuration estimation can lead to the peg getting stuck midway through insertion. 
This motivates the need for contact-aware and adaptive calibrated uncertainty estimates for planning.
\newline\indent
\begin{figure}[!t]
\centering\includegraphics[width=\linewidth]{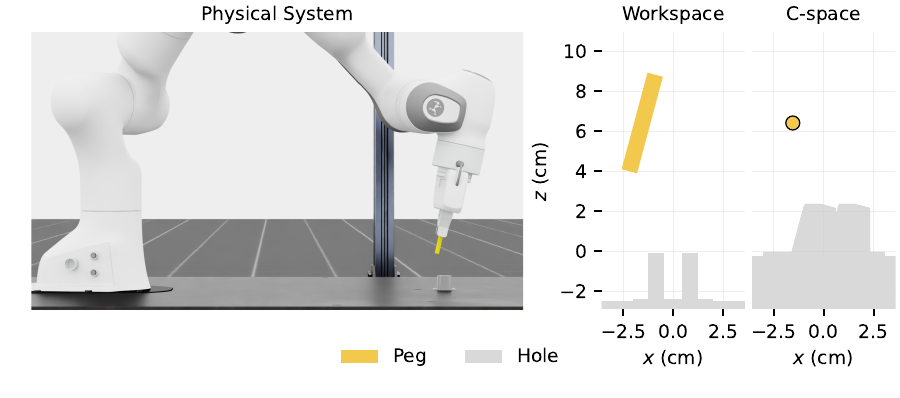}
    \vspace{-10mm}
    \caption{Tight-tolerance peg insertion in Isaac Sim. Left: Simulator view showing the 7-DoF manipulator, rounded peg (yellow), and hole fixture (gray). We restrict the end-effector motion to the hole's plane, reducing the space of possible peg poses from $SE(3)$ to $SE(2)$. Middle: Workspace view, where $(x,z)$ indicates the peg's centroid location relative to the hole, and $\theta$ indicates its in-plane orientation. Right: Cross-sectional view of the three-dimensional C-space used for motion planning, in which the peg is a point and obstacles are inflated according to the peg's shape and $\theta$ (not visualized). We show obstacle deformation for $\theta=-15^\circ$. Appendix~\ref{app:isaaclabImpDetail} shows obstacle deformations for other peg orientations. At this orientation, the peg cannot enter the hole.}
    \label{fig:factory_pih}
\end{figure}
To stratify the peg's $SE(2)$ C-space, we heuristically label each configuration according to a ``contact fingerprint.''
We first decompose the rectangular cross-section of the cylindrical peg into its four vertices and four edges.
To robustly detect contact between the peg and the environment, we place a disk of radius $0.75$ mm around each vertex and a tube of radius $0.75$ mm around each edge. Then, for each query configuration $c\in SE(2)$, we define the binary contact fingerprint $\phi(c)=(v_1,v_2,v_3,v_4,e_1,e_2,e_3,e_4)\in\{0,1\}^8$, where $v_i$ and $e_i$ denote the binary contact values of the $i$-th vertex and edge, respectively. Configurations with identical contact fingerprints are assigned to the same stratum.
As for the marble environment, we discretize the C-space offline and precompute the stratum indexer $\indexer$ on this grid using the binary contact fingerprints (see Appendix~\ref{app:isaaclabImpDetail} for details). At planning time, grid lookup returns $m=\indexer(\config)$, the index of the stratum $S_m$ containing $\config$. We construct $\calibset$ by again performing single-step rollouts with $\trueDynamics$ using sampled feasible configurations and actions.
We first sample initial configurations over the set of all strata in the $SE(2)$ C-space grid. For each $\config$, a control input $\action \in [\actionMin, \actionMax]$ is sampled uniformly. We repeat this process until we obtain $\lvert \calibset \rvert = 70{,}000$ transitions.
All particle-based methods sample $L=8$ particles.
Because the peg's configuration lies in $SE(2)$, we compute distances between particles and true configurations in the Lie algebra.
Given two poses $T_1,T_2 \in SE(2)$, the relative transform $T_\Delta = T_1^{-1}T_2$ has matrix logarithm $\log(T_\Delta)\in\mathfrak{se}(2)$, yielding the twist $(\rho_x, \rho_z, \omega)$.
We use the distance $\lVert(\rho_x, \rho_z, \ell\omega)\rVert_2$, where $\ell=0.0253$ m is the peg characteristic length; for small $\lvert \omega \rvert$, $\ell\lvert \omega \rvert$ approximates the maximum rotation-induced displacement of any point on the peg.

\looseness-1 \textbf{Numerical coverage validation.}
To verify the group-conditional coverage guarantee in Theorem~\ref{thm:stratcp} and the marginal coverage guarantee in Corollary~\ref{cor:margcov}, we collect a separate set of $90{,}000$ validation transitions, following a similar procedure to that used to construct $\calibset$. The coverage and volume are then estimated as in the previous section. Table~\ref{tab:pih_coverage} shows the results.
Because the scalar stratum index supplied to the \dtree encodes one of many binary contact fingerprints, we report in Table~\ref{tab:pih_coverage} a smaller, non-exhaustive set of interpretable categories: \emph{Point Contact} refers to configurations with zero-dimensional active contact support, whereas \emph{Edge Contact} refers to configurations with nonzero-length active contact support (cf. Figure~\ref{fig:contactmodesPeg} for examples and their corresponding reported contact labels). These categories do not include all strata; e.g., cases with multiple contact points appear only in the \emph{Aggregate} row.\newline\indent
\looseness-1 All methods approximately achieve marginal coverage, as expected. Although \lucca and \ablationNoStrata adapt to the current configuration and action, they and \pcp undercover in Free Space, likely because none of these baselines accounts for the next-configuration stratum. The two \method variants are the only methods that achieve both marginal and sufficient coverage within each reported contact category. This suggests that our approach can improve transition uncertainty quantification across complex configurations and contact modes.

\input{pih_coverage_table.tex}

\begin{figure}[!b]
    \centering
    \vspace{-5mm}
    \centering\includegraphics[width=\linewidth]{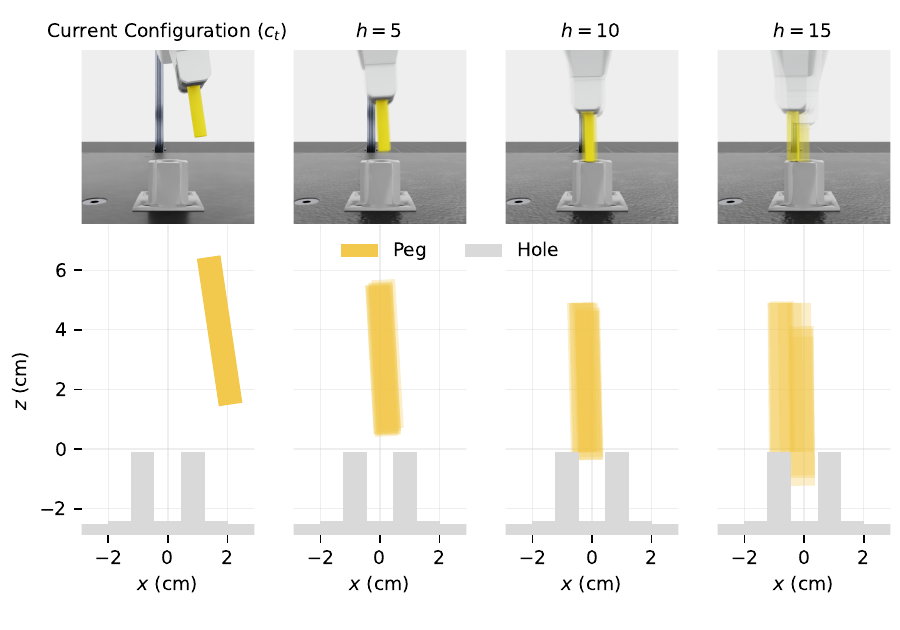}
    \vspace{-9mm}
    \caption{\looseness-1 Propagation of $\approxDynamics$'s particle predictions at relative planning-horizon steps $h=5,\, 10,\,15$. Top: Isaac Sim renderings of the current peg configuration and estimated future configurations. Bottom: Workspace view of the same particles.}
    \label{fig:plan_horizon}
\end{figure}

\textbf{Probabilistic motion planning.} We again use MPPI for trajectory optimization in the peg-insertion task. The objective combines the Factory peg-insertion cost, a reward for overlap with goal-compatible contact configurations, and a contact-aware control penalty. The complete objective is provided in Appendix~\ref{app:isaaclabImpDetail}.
Unlike in the marble example, for the peg-in-hole task we construct the calibrated prediction region only for the first step of each MPC plan. For later plan steps, we simply propagate the particles using $\approxDynamics$.
This is done partly for computational reasons and partly because the manipulator system is quasi-static and our \method coverage guarantees apply exclusively to single-step transitions.
For the more dynamic marble task, uncertainty can grow significantly along a planning horizon, requiring reliable multistep uncertainty estimates. In quasi-static settings, one-step errors are more forgiving and can be corrected more easily in the next MPC plan.
A visualization of the $\approxDynamics$ particles along the planning horizon is shown in Figure~\ref{fig:plan_horizon}.
Videos comparing the trajectories generated by \method and the baselines are provided on the project website.
We evaluate planning performance by starting from $50$ feasible initial configurations selected using a deterministic Halton sequence shown in Figure~\ref{fig:haltonPoses}. Table~\ref{tab:pih_inference_rand} reports the average number of steps taken to full insertion and success rates within a 75-step (5 s) limit.\newline\indent

\input{pih_inference_rand_table}

\looseness=-2 The results indicate that reliable uncertainty calibration can improve task success in tight-tolerance settings under significant model mismatch and aleatoric disturbances. The conformal prediction-based methods outperform the \particleNoCP baseline across the tested initial conditions. Furthermore, the improved performance of \method relative to the other CP methods suggests that our ability to represent uncertainty in stratified configuration spaces can improve downstream task success rate and efficiency/speed. These results also demonstrate that producing uncertainty estimates that are highly sensitive to individual predicted particles ($\nearestNeighborIndex=1$ ablation) can degrade planning performance.
Further, adequate coverage in each contact regime does not necessarily imply improved task performance. Empirically, the smoother and less spurious prediction regions produced by $\nearestNeighborIndex=4$ were easier to optimize over with the tested MPPI parameters. Our results suggest that treating robot motion uncertainty as action-, configuration-, and stratum-dependent can yield informative \transdimensional uncertainty representations for downstream probabilistic motion planning.

\vspace{-4mm}
\subsection{Limitations}
\vspace{-3mm}
\looseness=-3 Our guarantees are single-step and, although we observe improved closed-loop performance, we do not prove multistep closed-loop coverage.
\method assumes knowledge of the environmental constraints and access to a correct stratum indexer, since geometric uncertainty could produce prediction regions containing infeasible configurations.
Empirically, a sufficiently large $\lvert\calibset\rvert$ is needed to capture local uncertainty variations without sparsely populated Mondrian groups.
Task performance depends on hyperparameters: some \dtree parameters produce sparse partitions with volatile conformal thresholds, whereas broad partitions may adapt poorly. Similarly, different MPPI cost weights can lead to myopic or overconservative trajectories.
We leave a rigorous analysis of how performance varies with $\lvert \calibset \rvert$ for future work.
\vspace{-4mm}
\section{Conclusion}
\vspace{-3mm}
\looseness=-3 We proposed a state-action-stratum-aware conformal prediction algorithm to construct probabilistically valid next-configuration prediction regions in stratified configuration spaces.
Our approach combines particle-based dynamics predictors with a Mondrian-based adaptive calibration to produce uncertainty sets that can capture multimodal, discontinuous, and \transdimensional robotic motion uncertainty and do not include infeasible regions of the robot's configuration space. 
We proved and numerically validated \method's finite-sample coverage guarantees for single-step predictions. 
Practically, our adaptive uncertainty estimation procedure can improve task success, relative to relevant particle-based and adaptive baselines, in both a dynamic marble control task and a tight-tolerance peg-insertion task, even under significant model mismatch and aleatoric disturbances.

\vspace{-4mm}
\acks{\vspace{-2mm}This work was supported in part by the Office of Naval Research Grant N00014-24-1-2036 and NSF grants IIS-2113401 and IIS-2220876. The authors declare no competing interests.}

\bibliography{sparcp}

\appendix

\section{Additional Proofs}\label{app:proofs}

\margcov*
\begin{proof}
Let $G := g(\CPinputTest,\indexer(\CPoutputTest))$.
The mapping $y\mapsto g(\CPinputTest,\indexer(y))$ partitions the output space into disjoint groups indexed by $\groupIndex$. Hence the events $\{G=\groupIndex\}_{\groupIndex=1}^{\numGroups}$ are mutually exclusive and exhaustive.
By the law of total probability,
\[
\mathbb P\{\CPoutputTest \in \CPregion(\CPinputTest)\}
=
\sum_{\groupIndex=1}^{\numGroups}
\mathbb P\{\CPoutputTest \in \CPregion(\CPinputTest)\mid G=\groupIndex\}
\mathbb P\{G=\groupIndex\}.
\]
For every $\groupIndex$ with $\mathbb P\{G=\groupIndex\}>0$, Theorem~\ref{thm:stratcp} gives
\[
\mathbb P\{\CPoutputTest \in \CPregion(\CPinputTest)\mid G=\groupIndex\}
\ge 1-\failureRate.
\]
Terms with $\mathbb P\{G=\groupIndex\}=0$ vanish from the sum. 
Therefore,
\[
\mathbb P\{\CPoutputTest \in \CPregion(\CPinputTest)\}
\ge
\sum_{\groupIndex=1}^{\numGroups} (1-\failureRate)\mathbb P\{G=\groupIndex\}
=
1-\failureRate.
\]
\end{proof}

\vspace{-10mm}
\section{Additional Experimental Details}\label{apd:extraexp}

\subsection{Marble Labyrinth Dynamics}\label{app:marbledyn}

For both approximate particle propagation and true dynamics, the commanded action $\actionAt{t}$ is perturbed and clipped before execution,
yielding the bounded stochastic action
\begin{equation}
\boundedActionAt{t}
=
\mathrm{clip}(\actionAt{t}+\noise,\;\actionMin,\actionMax),
\qquad \text{where } 
\noise\sim\mathcal{N}(0, 0.0025 I_{2}).
\end{equation}
All methods have access to the approximate dynamics model $\approxDynamics$ below, from which they can sample particles; $\gravity$ denotes gravitational acceleration:
\begin{equation}\label{eq:approximateMarbleStochastic}
\underbrace{
\begin{bmatrix}
x^b_{t+1} \\
\dot{x}^b_{t+1} \\
y^b_{t+1} \\
\dot{y}^b_{t+1} \\
\alpha_{t+1} \\
\beta_{t+1}
\end{bmatrix}
}_{\nextState}
=
\underbrace{
\begin{bmatrix}
1 & \dt & 0 & 0 & 0 & 0 \\
0 & 1 & 0 & 0 & -\frac{5}{7}\gravity\dt & 0 \\
0 & 0 & 1 & \dt & 0 & 0 \\
0 & 0 & 0 & 1 & 0 & -\frac{5}{7}\gravity\dt \\
0 & 0 & 0 & 0 & 1 & 0 \\
0 & 0 & 0 & 0 & 0 & 1
\end{bmatrix}
}_{A_{\mathrm{approx}}}
\underbrace{
\begin{bmatrix}
x^b_{t} \\
\dot{x}^b_{t} \\
y^b_{t} \\
\dot{y}^b_{t} \\
\alpha_t \\
\beta_t
\end{bmatrix}
}_{\state}
+
\underbrace{
\begin{bmatrix}
0 & 0 \\
0 & 0 \\
0 & 0 \\
0 & 0 \\
\tiltGainAlpha\dt & 0 \\
0 & \tiltGainBeta\dt
\end{bmatrix}
}_{B}
\boundedActionAt{t}
\end{equation}
Yet, the true system evolves according to the mismatched dynamics
\begin{equation}\label{eq:realMarbleStochastic}
\nextState
=
\underbrace{
\begin{bmatrix}
1 & \dt & 0 & 0 & 0 & 0 \\
0 & 1 & 0 & 0 & -\frac{5}{7}\gravity\dt \cdot 2& 0 \\
0 & 0 & 1 & \dt & 0 & 0 \\
0 & 0 & 0 & 1 & 0 & -\frac{5}{7}\gravity\dt \cdot 2 \\
0 & 0 & 0 & 0 & 1 & 0 \\
0 & 0 & 0 & 0 & 0 & 1
\end{bmatrix}
}_{A_{\mathrm{real}}}
\state
+
B \boundedActionAt{t}.
\end{equation}

The grid used to look up strata is precomputed with a resolution of $1$ mm. Edge strata are defined as cells within $2$ mm of an inflated C-space obstacle. Corner strata are also within $2$ mm of a C-space obstacle, but while edges are in locally smooth wall surfaces, corners are at non-smooth obstacle features such as concave vertices---there are multiple obstacle constraints acting on a robot at a corner. The full maze environment, along with the start and goal locations for each subsection, is shown in Figure~\ref{fig:marbleLabyrinth}.

\begin{figure}[!b]
    \centering
    \vspace{-10mm}
\includegraphics[width=.72\textwidth]{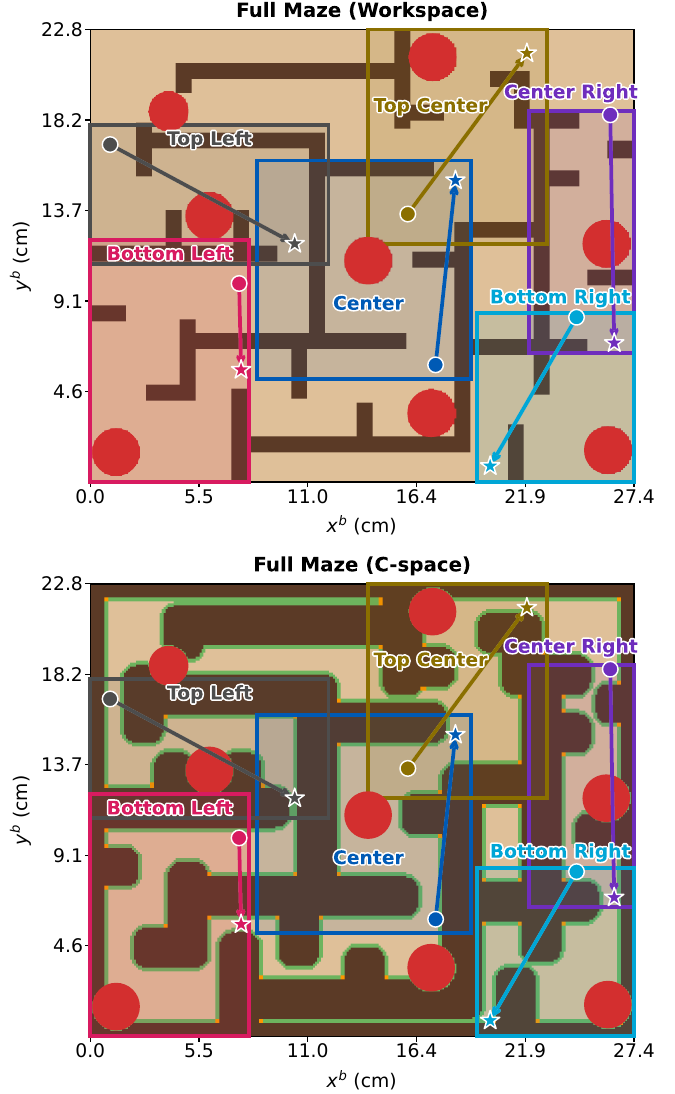}
\vspace{-4mm}
    \caption{Maze inspired by the BRIO 34030 marble labyrinth, with pits (red) enlarged and moved toward more central locations. The subsections used for motion planning evaluation are shown, along with their start positions (circles) and goals (stars).}
    \label{fig:marbleLabyrinth}
\end{figure}

The test case transitions collected for numerically validating coverage were obtained by enumerating the same position, velocity, board angle, and action Cartesian grid as for $\calibset$ collection.

\subsection{Marble Planning Implementation}\label{app:mppiParams}
For both tasks, we used Model Predictive Path Integral Control (MPPI) as the trajectory optimizer. For the marble control task, at each step MPPI sampled $4096$ control sequences over a horizon of $H=6$ steps, corresponding to $0.6$ s of future motion.
We used temperature $\lambda=0.2$ and an action perturbation covariance of $\Sigma_{\textrm{MPPI}}=\begin{bmatrix}1.0&0.1\\0.1&1.0\end{bmatrix}$; perturbed controls were clipped to remain within the actuation bounds.
Each sampled rollout was evaluated using the objective function
\vspace{-2mm}
\begin{equation}
  \costObjective =
  50\, d_{\mathrm{trav}}(\bar p_H, p_{\goal})+
  \sum_{h=1}^{H}
  \left[
  30\, d_{\mathrm{trav}}(\bar p_h, p_{\goal})
  +
  250\, n_{\mathrm{pit}}(\CPregion_h)
  \right],
  \vspace{-2mm}\label{eq:mppiCOstMarble}
\end{equation}
where $d_{\mathrm{trav}}$ is the traversable-path distance to the goal and $\bar p_h$ is the mean particle configuration.
While Euclidean distance does not distinguish between collision-free paths and paths through infeasible configurations, traversable-path distance considers only collision-free paths in C-space. We precomputed $d_{\mathrm{trav}}$ for each map--$\goal$ combination, enabling efficient lookup during planning.
To penalize unsafe motions, we included $n_{\mathrm{pit}}(\CPregion_h)$ --- the number of discretized C-space cells that are both in the prediction region $\CPregion_h$ and the known pit region --- which provides a larger penalty when a larger portion of $\CPregion$ is unsafe.

\subsection{Extra Marble Control Results}\label{app:marbleMoreResults}
In Table~\ref{tab:marblecoveragebymap}, we report empirical coverage and estimated volumes per map in the maze environment. In Figure~\ref{fig:marble_planning}, we visualize the trajectories obtained by each method in the six tested maze subsections.

\input{marble_testcasesPerMap_table}

\begin{figure}[!t]
    \centering
    \includegraphics[width=\linewidth]{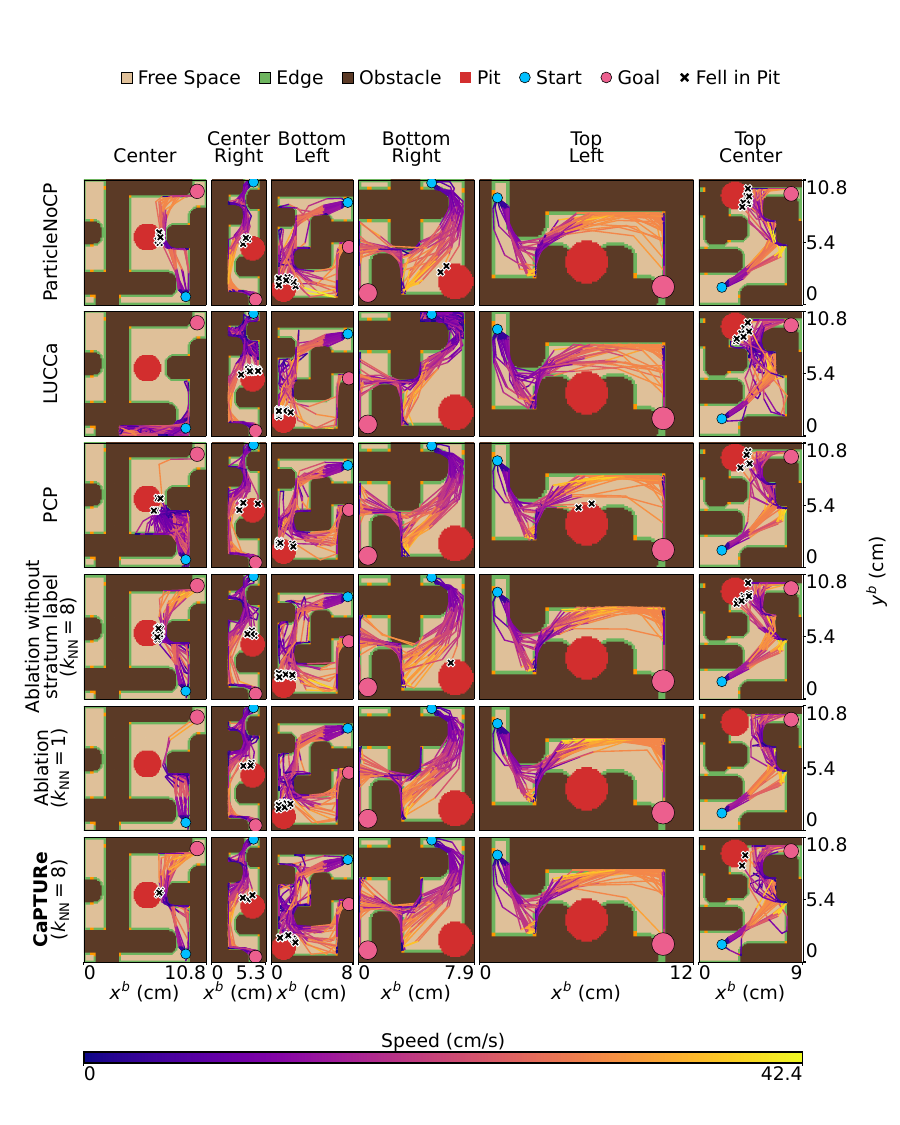}
    \caption{Rollouts of all methods across the six evaluated marble maze sections. }
    \label{fig:marble_planning}
\end{figure}

\FloatBarrier
\subsection{Isaac Sim Peg Insertion Implementation Details}\label{app:isaaclabImpDetail}

Isaac Lab Factory's default ``Peg Insertion'' task uses a 25-mm-tall receptacle with a 1-mm bevel around the circumference of the hole.
The beveled lip makes insertion substantially easier. We used a 24-mm-tall receptacle without a beveled hole.
In our experiments, we use a grid over the $SE(2)$ C-space of the peg to streamline dataset collection and define the set of poses over which we evaluate the prediction region during planning.
We construct the grid by discretizing the configuration space using: $x \in [-3, 3]$ cm and $z\in[-1, 5]$ cm, with $d_x=d_z=0.25$ mm, and $\theta \in [-40^\circ, 40^\circ]$, with $d_\theta=0.5^\circ$.
The grid cell $(x_u,z_v,\theta_w)$ corresponds to a peg pose with its centroid at $(x_u,z_v)$ and orientation $\theta_w$.
Each grid cell is analytically checked for feasibility. To account for small interpenetrations allowed by PhysX, we treat poses with a peg--environment overlap area of at most $0.01\,\mathrm{mm}^2$ (in the $xz$ plane) as feasible. Each feasible cell is assigned a stratum index using its binary contact fingerprint.

  Let $\widetilde\Psi_t$ denote the Factory insertion cost~\citep{factory},
  averaged over the first-step prediction region at $t=0$ and over the
  propagated particle cloud thereafter. Let $\rho_0$ denote the fraction of the
  first-step prediction region outside the goal stratum, and let $h_0$ denote
  its average normalized Hamming distance to the goal-contact fingerprint. With
  $\gamma=0.95$, $H=15$, and $[x]_+=\max\{x,0\}$, we use the following MPPI
  objective:
  \vspace{-3mm}
  \begin{equation*}
  \costObjective
  =
  \left(
  \sum_{t=0}^{H}
  \gamma^t
  \left[
  \widetilde\Psi_t
  +0.005\lVert \actionAt{t}\rVert_2^2
  \right]
  \right)
  +20\rho_0
  +0.6[h_0-0.5]_+\lVert \actionAt{0}\rVert_2^2 ,\vspace{-3mm}
  \end{equation*}
  with 128 samples and temperature $\lambda=0.1$.
  \looseness-1 Since the terminal state has no associated action,
  $\actionAt{H}=\mathbf{0}$. 
  As
  \particleNoCP does not produce a prediction region, we compute the first-step region statistics over an axis-aligned 
 C-space box fitted to its particles.

\begin{figure}[!hb]
    \centering
\centering\includegraphics[width=\textwidth]{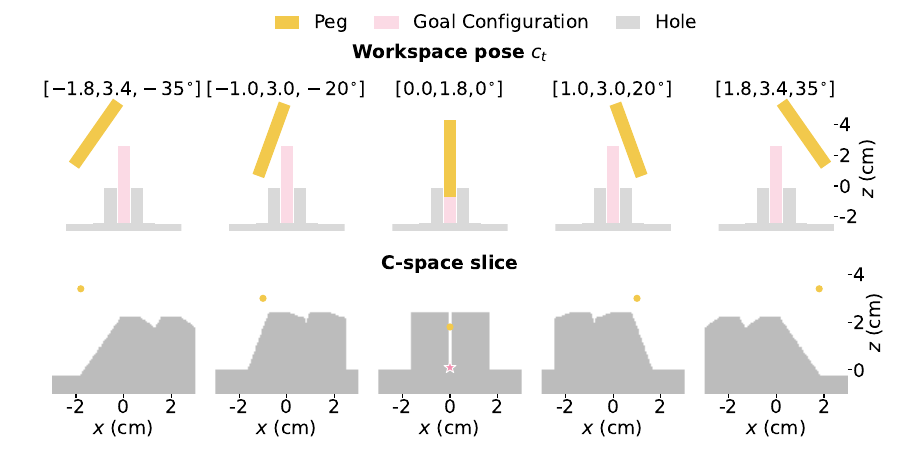}
\vspace{-8mm}
    \caption{\looseness-1 Workspace views and C-space slices for different peg orientations. For orientations in which the peg cannot be inserted, the hole appears blocked in the corresponding C-space slice. We plan in C-space, treating the peg as a point, and build a 3D C-space discretization offline. The third dimension, which is not visible in each slice, is the peg orientation $\theta$. The orientation increases from left to right, and the goal lies in the C-space slice with $\theta=0^\circ$.}
    \label{fig:cspaceviewPeg}
\end{figure}

\noindent\begin{minipage}{\textwidth}

    \centering
\centering\includegraphics[width=.9\textwidth]{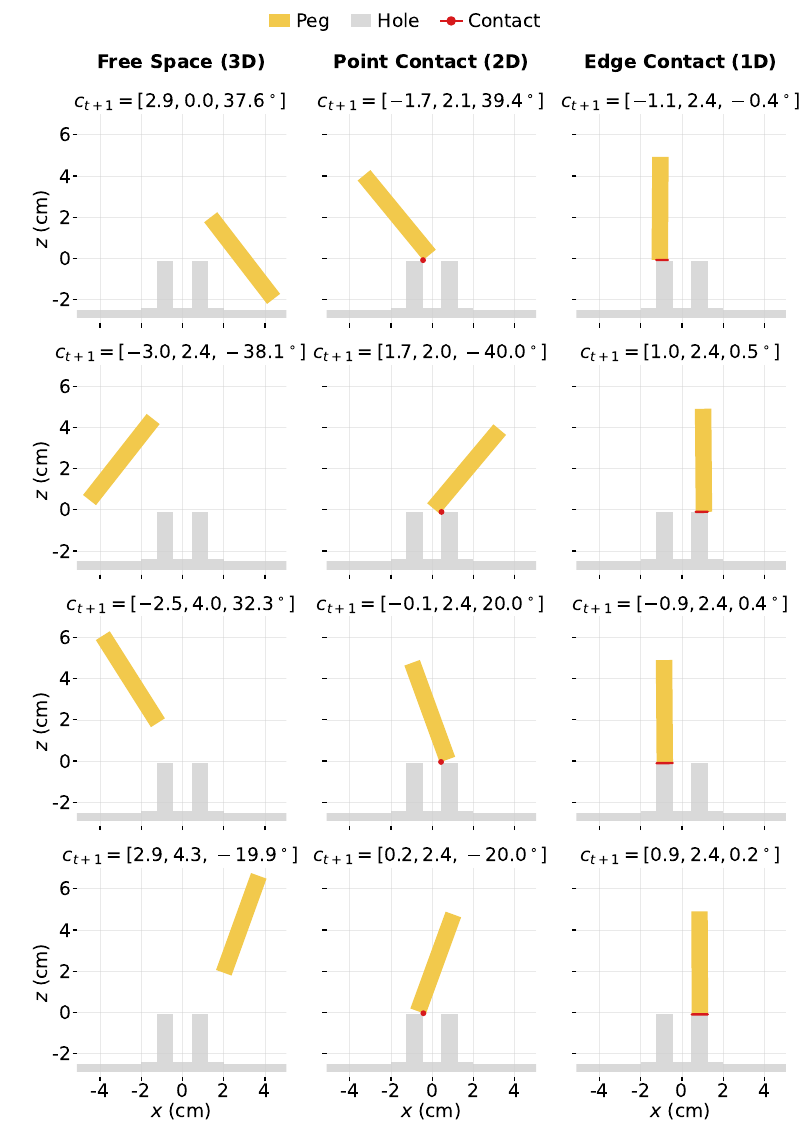}
\vspace{-4mm}
    \captionof{figure}{\looseness-1 Workspace views of some next configurations $\nextConfig$ from the $90{,}000$ validation transitions used to estimate coverage numerically for the peg-in-hole task. The reported ``interaction regimes''---Free Space, Point Contact, and Edge Contact---are used for interpretability. The \dtree receives the scalar stratum index $m=\indexer(\nextConfig)$.}
    \label{fig:contactmodesPeg}
\end{minipage}

\noindent\begin{minipage}{\textwidth}
    \centering
\centering\includegraphics[width=\textwidth]{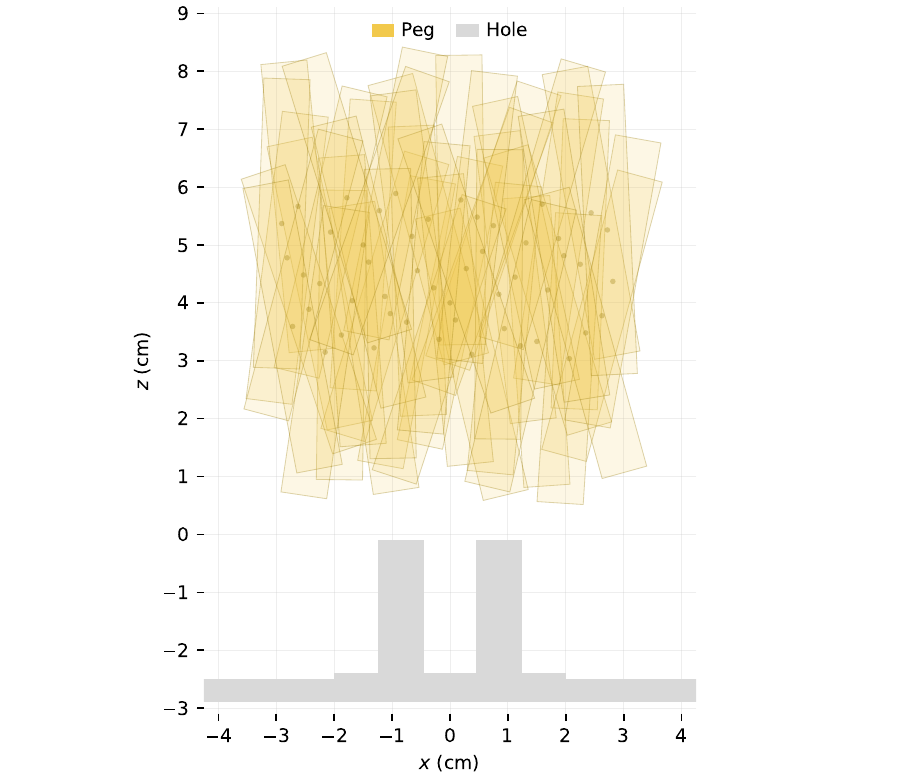}
\vspace{-4mm}
    \captionof{figure}{\looseness-1 Initial peg configurations (one per rectangle) used for the planning evaluation in Table~\ref{tab:pih_inference_rand}, obtained from a deterministic Halton sequence.}
    \label{fig:haltonPoses}
\end{minipage}

\end{document}

%% file: preamble.tex
\usepackage{amsfonts}
\usepackage{amsmath}
\usepackage{mathrsfs}
\usepackage{wrapfig}
\usepackage{xspace}
\usepackage{multirow}
\usepackage{mathtools}
\usepackage{thmtools,thm-restate}
\usepackage{array}
\usepackage{booktabs}
\usepackage{layouts}
\usepackage[hypcap=false]{caption}
\usepackage{siunitx}
\usepackage{placeins}
\usepackage{graphicx}
\usepackage{microtype} 
\usepackage{upgreek}

\usepackage{pifont} 

\usepackage{makecell}

\usepackage{dirtytalk} 
\usepackage{enumitem}
\usepackage{hyperref} %
\usepackage{icomma} 

\newcommand{\projWebsite}{\href{https://um-arm-lab.github.io/capture}{\texttt{https://um-arm-lab.github.io/capture}}\xspace}
\newcommand{\transdimensional}{trans-dimensional\xspace}

\newcolumntype{G}{>{\columncolor{green!15}}c} 
\newcolumntype{R}{>{\columncolor{red!15}}c} 
\newcolumntype{O}{>{\color{Orange}\bfseries\boldmath}c} 
\newcommand{\resultHeader}[1]{{\normalcolor\mdseries\unboldmath #1}}

\newcommand{\superConfig}{\mathfrak{c}_{t}}
\newcommand{\superVelocity}{\dot{\mathfrak{c}}_{t}}
\newcommand{\superCspace}{\mathfrak{C}}
\newcommand{\nextSuperConfig}{\mathfrak{c}_{t+1}}
\newcommand{\config}{c_{t}}

\newcommand{\nextConfig}{{c}_{t+1}}

\newcommand{\Cspace}{C}
\newcommand{\feasibleCspace}{C_{\mathrm{feas}}}
\newcommand{\obstacleCspace}{C_{\mathrm{obs}}}

\newcommand{\state}{s_t}
\newcommand{\nextState}{s_{t+1}}
\newcommand{\initState}{s_{0}}
\newcommand{\goal}{\texttt{Goal}\xspace}
\newcommand{\failure}{\texttt{Fail}\xspace}
\newcommand{\dt}{\Delta t}
\newcommand{\noise}{\varepsilon_t}
\newcommand{\actionAt}[1]{a_{#1}}
\newcommand{\action}{\actionAt{t}}
\newcommand{\boundedActionAt}[1]{\bar a_{#1}}
\newcommand{\actionMin}{a_{\min}}
\newcommand{\actionMax}{a_{\max}}
\newcommand{\actionSpace}{\mathcal A}
\newcommand{\trueDynamics}{f}
\newcommand{\approxDynamics}{\hat f}

\newcommand{\gravity}{g_{\mathrm{grav}}}
\newcommand{\CPinput}{X}
\newcommand{\CPinputTest}{X_{n+1}}
\newcommand{\CPinputSpace}{\mathcal X}
\newcommand{\CPoutput}{Y}
\newcommand{\CPoutputTest}{Y_{n+1}}
\newcommand{\CPoutputSpace}{\mathcal Y}
\newcommand{\predParticleCloud}[1]{\hat{\mathbf Y}_{#1}}
\newcommand{\predParticle}[2]{\hat Y_{#1,#2}}
\newcommand{\scoreFunc}{r}
\newcommand{\scoreVal}{R}

\newcommand{\calibset}{D_{\mathrm{cal}}}
\newcommand{\trainset}{D_{\mathrm{train}}}
\newcommand{\augset}{\bar{D}_{\mathrm{cal}}}
\newcommand{\pcpAugset}{\widetilde{D}_{\mathrm{cal}}}
\newcommand{\numcal}{n}
\newcommand{\CPregion}{\hat{\mathcal C}}
\newcommand{\CPthreshold}{\hat q}
\newcommand{\strata}{S_m}

\newcommand{\groupIndex}{j}
\newcommand{\numGroups}{J}
\newcommand{\nearestNeighborIndex}{k_{\mathrm{NN}}}
\newcommand{\costObjective}{\operatorname{Cost}}
\newcommand{\captureAblationHeader}[1]{\makecell[c]{Ablation\\($\nearestNeighborIndex=#1$)\\}}
\newcommand{\capturePrimaryHeader}[1]{\makecell[c]{\textbf{\method}\\($\nearestNeighborIndex=#1$)}}
\newcommand{\dtree}{\texttt{DTree}\xspace}
\newcommand{\tiltGainAlpha}{\gamma_\alpha}
\newcommand{\tiltGainBeta}{\gamma_\beta}

\newcommand{\failureRate}{\upalpha}

\newcommand{\indexer}{\mathcal T}
\newcommand{\pred}{\hat f}
\newcommand{\densityEstimate}{\hat p (\CPoutput \mid \CPinput)}

\newcommand{\predspace}{\mathcal H}

\newtheorem{assumption}{Assumption}
\newcommand{\transpose}{^\mathsf{T}}

\newcommand{\method}{\text{CaPTURe}\xspace}
\newcommand{\methodNameLong}{\textbf{Ca}librated \textbf{P}article-sets for \textbf{T}rans-dimensional \textbf{U}ncertainty \textbf{Re}presentation\xspace}

\newcommand{\particleNoCP}{{ParticleNoCP}\xspace}
\newcommand{\particleNoCPHeader}{\makecell[c]{Particle\\NoCP}}
\newcommand{\splitcp}{{SplitCP}\xspace}
\newcommand{\lucca}{{LUCCa}\xspace}
\newcommand{\pcp}{{PCP}\xspace}
\newcommand{\ablationNoStrata}{Ablation w/o stratum label\xspace}
\newcommand{\ablationNoStrataHeader}[1]{\makecell[c]{Ablation\\w/o stratum\\label ($\nearestNeighborIndex=#1$)}}
\newcommand{\avgStepsHeader}{\makecell[l]{Avg. Steps\\(mean$\pm$std) $\downarrow$}}



\makeatletter
\@ifundefined{proof}{%
\@ifundefined{IEEEproof}{%
\newenvironment{proof}[1][Proof]{%
\par\noindent\textit{#1: }\ignorespaces
}{%
\hspace*{\fill}$\square$\par
}%
}{%
\newenvironment{proof}{\IEEEproof}{\endIEEEproof}%
}%
}{}
\makeatother

%% file: algo1.tex
\begin{algorithm2e}[!t]
\caption{\small \methodNameLong}
\label{alg:SParCP}

\LinesNumbered
\SetAlgoNlRelativeSize{0}
\SetNlSty{algolinenum}{}{}
\SetNlSkip{-1.5em}
\SetInd{0.8em}{1em}
\SetAlgoNoLine
\DontPrintSemicolon
\Indp
\KwIn{$\approxDynamics,\; \calibset,\; \failureRate,\; \costObjective, \nearestNeighborIndex, \goal, \failure, \indexer, H, \{S_m\}_{m\in M}, L$}

\BlankLine
\tcc{Offline: iterate over the full calibration set $\calibset$}
\For{$(\CPinput_i,\CPoutput_i)\in \calibset$}{
    $\predParticleCloud{i} \leftarrow \{\predParticle{i}{1},\dots,\predParticle{i}{L}\},\;
    \predParticle{i}{l}\sim \approxDynamics(\CPinput_i)$
    \hfill \tcp*{Sample $L$ particles from $\pred(\CPinput_i)$}
    $\scoreVal_i \leftarrow \mathrm{kNN}_{\nearestNeighborIndex}(\CPoutput_i,\predParticleCloud{i})$
    \hfill \tcp*{Compute nonconformity score using Equation~\eqref{eq:knn}}
}

$\augset \leftarrow \{(\CPinput_i,\indexer(\CPoutput_i),\scoreVal_i):(\CPinput_i,\CPoutput_i)\in\calibset\}$\;
$\augset^{part},\augset^{cp}\leftarrow\texttt{RandomSplit}(\augset)$\;
\tcc{Offline: Fit stratum-aware decision tree with $\augset^{part}$}
$\texttt{DTree}
\leftarrow
\texttt{CART}(\augset^{part})$
\hfill \tcp*{Partition $\CPinputSpace \times M$ to $\min \text{Var}(\scoreVal)$}
\tcc{Offline: perform SplitCP on each \texttt{DTree}-induced group using $\augset^{cp}$}
$\groupIndex_i \leftarrow \texttt{DTree}(\CPinput_i,\indexer(\CPoutput_i)),\ \ \forall(\CPinput_i,\indexer(\CPoutput_i),\scoreVal_i)\in\augset^{cp}$
\hfill\tcp*{Assign calibration data to groups}

$\hat{q}_{\groupIndex} \leftarrow \texttt{SplitCP}(\{\scoreVal_i\}_{i\in \mathcal I_{\groupIndex}};\failureRate),\ \ \forall \groupIndex\in\{1,\ldots,\numGroups\}$
\label{algLine:qhatsplitcp}
\hfill\tcp*{Get threshold per leaf node}

\BlankLine
\tcc{Online: Model Predictive Control (MPC) with $\method$. Apply $a^*_t$}

\While{$s_t\notin\goal \land s_t\notin\failure$}{

\tcp{\texttt{MPPI} calls \texttt{CP\_REGION} with $(s_\tau,a_\tau)$ and propagated particles; for $\tau>t$, set $s_\tau$ to the particle-cloud mean.}

$a_{t:t+H-1}^*
\leftarrow
\texttt{MPPI}\big(
s_t,\approxDynamics,
\texttt{CP\_REGION},
\texttt{DTree},\{\hat q_{\groupIndex}\}_{j=1}^J,\nearestNeighborIndex,\{S_m\}_{m\in M},
\costObjective,\goal,H
\big)$\;
}

\end{algorithm2e}

%% file: algo2.tex
\begin{algorithm2e}[!h]
\caption{\small \texttt{CP\_REGION} (Stratified CP Region Construction, cf. Figure~\ref{fig:inference})}
\label{alg:rollout_SPar-CP}

\LinesNumbered
\SetAlgoNlRelativeSize{0}
\SetNlSty{algolinenum}{}{}
\SetNlSkip{-1.5em}
\SetInd{0.8em}{1em}
\SetAlgoNoLine
\DontPrintSemicolon
\Indp

\KwIn{$s_\tau,a_\tau,\predParticleCloud{\tau+1},\texttt{DTree},\{\hat q_{\groupIndex}\},\nearestNeighborIndex,\{S_m\}_{m\in M}$}

\BlankLine
$\CPinput_\tau \leftarrow (s_\tau,a_\tau)$
\hfill \tcp*{State-action pair at MPPI rollout step $\tau$}

\BlankLine
\tcc{Construct one prediction region per candidate future stratum}
\For{$m\in M$}{
    $\groupIndex_m \leftarrow \texttt{DTree}(\CPinput_\tau,m)$
    \hfill \tcp*{Get group for $\CPinput_\tau$ + stratum index $m$}

    $\CPregion_{\tau+1}^{(m)}
    \leftarrow
    \{y\in S_m:
    \mathrm{kNN}_{\nearestNeighborIndex}(y,\predParticleCloud{\tau+1})
    \le
    \hat q_{\groupIndex_m}\}$
    \hfill \tcp*{Build region in $S_m$}
}

$\CPregion_{\tau+1}
\leftarrow
\bigcup_{m\in M}\CPregion_{\tau+1}^{(m)}$
\hfill \tcp*{Union over candidate C-space strata}

\textbf{return} $\CPregion_{\tau+1}$
\end{algorithm2e}

%% file: marble_testcases_table.tex
\newsavebox{\marbleCoverageTableBox}
\begin{table*}[!t]
\centering
\footnotesize
\setlength{\tabcolsep}{3pt}
\renewcommand{\arraystretch}{1.12}
\captionsetup{skip=2pt,width=\textwidth,font=footnotesize}
\caption{Marble labyrinth empirical coverage and C-space volume (over 427,680 test cases).}
\label{tab:marblecoverage}
\sbox{\marbleCoverageTableBox}{%
\begin{tabular}{l!{\vrule width 0.45pt}l!{\vrule width 0.45pt}ccccc}
\toprule
\multicolumn{1}{c!{\vrule width 0.45pt}}{\makecell[c]{Metric}} & \makecell[c]{Stratum} & \lucca & \pcp & \ablationNoStrataHeader{8} & \captureAblationHeader{1} & \capturePrimaryHeader{8} \\
\midrule
\multirow{4}{*}{\makecell[c]{Empirical\\Coverage (\%) $\uparrow$}} & All Strata & 90.9 & 90.0 & 90.1 & 90.2 & 90.2 \\
 & \cellcolor{black!8}Free Space (2D) & \cellcolor{black!8}83.4 & \cellcolor{black!8}76.4 & \cellcolor{black!8}83.7 & \cellcolor{black!8}90.1 & \cellcolor{black!8}90.0 \\
 & Edge (1D) & 94.4 & 96.4 & 93.1 & 89.9 & 90.2 \\
 & \cellcolor{black!8}Corner (0D) & \cellcolor{black!8}96.5 & \cellcolor{black!8}99.0 & \cellcolor{black!8}94.8 & \cellcolor{black!8}92.4 & \cellcolor{black!8}90.6 \\
\midrule
\makecell[c]{Avg. C-space\\Volume (ratio) $\downarrow$} & All Strata & 1.48 & 0.71 & 0.86 & 1.05 & 1.00 \\
\bottomrule
\end{tabular}
}
\usebox{\marbleCoverageTableBox}
\par\vspace{3pt}
\begin{minipage}{\wd\marbleCoverageTableBox}
\footnotesize
\raggedright
User-specified coverage is $90\%$. C-space volume is reported as a ratio relative to \textbf{\method}.\par
\end{minipage}
\vspace{-6mm}
\end{table*}

%% file: marble_planning_table.tex
\newsavebox{\marbleplanningtablebox}
\begin{table*}[!t]
\centering
\footnotesize
\setlength{\tabcolsep}{3pt}
\renewcommand{\arraystretch}{1.12}
\sbox{\marbleplanningtablebox}{%
\begin{tabular}{c!{\vrule width 0.45pt}l!{\vrule width 0.45pt}cccccc}
\toprule
\makecell[c]{Map} & \makecell[c]{Metric} & \particleNoCPHeader & \pcp & \lucca & \ablationNoStrataHeader{8} & \captureAblationHeader{1} & \capturePrimaryHeader{8} \\
\midrule
\multirow{4}{*}{\makecell[c]{Center}} & \cellcolor{black!8}Success (\%) $\uparrow$ & \cellcolor{black!8}53 & \cellcolor{black!8}3 & \cellcolor{black!8}0 & \cellcolor{black!8}\underline{57} & \cellcolor{black!8}7 & \cellcolor{black!8}\textbf{90} \\
 & \textcolor{black}{Fell in Pit} (\%) $\downarrow$ & 47 & 13 & 0 & 43 & 0 & 10 \\
 & \cellcolor{black!8}Timed Out (\%) $\downarrow$ & \cellcolor{black!8}0 & \cellcolor{black!8}84 & \cellcolor{black!8}100 & \cellcolor{black!8}0 & \cellcolor{black!8}93 & \cellcolor{black!8}0 \\
 & \avgStepsHeader & 9.9$\pm$1.3 & 8.0$\pm$0.0 & -- & 18.5$\pm$11.1 & 86.0$\pm$12.7 & 11.9$\pm$4.3 \\
\midrule
\multirow{4}{*}{\makecell[c]{Center\\Right}} & \cellcolor{black!8}Success (\%) $\uparrow$ & \cellcolor{black!8}83 & \cellcolor{black!8}\underline{87} & \cellcolor{black!8}57 & \cellcolor{black!8}83 & \cellcolor{black!8}\textbf{90} & \cellcolor{black!8}\textbf{90} \\
 & \textcolor{black}{Fell in Pit} (\%) $\downarrow$ & 17 & 13 & 40 & 17 & 10 & 10 \\
 & \cellcolor{black!8}Timed Out (\%) $\downarrow$ & \cellcolor{black!8}0 & \cellcolor{black!8}0 & \cellcolor{black!8}3 & \cellcolor{black!8}0 & \cellcolor{black!8}0 & \cellcolor{black!8}0 \\
 & \avgStepsHeader & 13.8$\pm$2.1 & 13.7$\pm$1.6 & 33.2$\pm$14.9 & 15.7$\pm$2.1 & 19.0$\pm$6.3 & 16.4$\pm$4.3 \\
\midrule
\multirow{4}{*}{\makecell[c]{Bottom\\Left}} & \cellcolor{black!8}Success (\%) $\uparrow$ & \cellcolor{black!8}40 & \cellcolor{black!8}57 & \cellcolor{black!8}17 & \cellcolor{black!8}47 & \cellcolor{black!8}\underline{67} & \cellcolor{black!8}\textbf{73} \\
 & \textcolor{black}{Fell in Pit} (\%) $\downarrow$ & 60 & 43 & 83 & 53 & 33 & 27 \\
 & \cellcolor{black!8}Timed Out (\%) $\downarrow$ & \cellcolor{black!8}0 & \cellcolor{black!8}0 & \cellcolor{black!8}0 & \cellcolor{black!8}0 & \cellcolor{black!8}0 & \cellcolor{black!8}0 \\
 & \avgStepsHeader & 16.2$\pm$3.0 & 16.4$\pm$3.3 & 25.4$\pm$5.0 & 18.5$\pm$7.2 & 16.9$\pm$3.6 & 23.1$\pm$13.6 \\
\midrule
\multirow{4}{*}{\makecell[c]{Bottom\\Right}} & \cellcolor{black!8}Success (\%) $\uparrow$ & \cellcolor{black!8}93 & \cellcolor{black!8}\textbf{100} & \cellcolor{black!8}67 & \cellcolor{black!8}\underline{97} & \cellcolor{black!8}\textbf{100} & \cellcolor{black!8}\textbf{100} \\
 & \textcolor{black}{Fell in Pit} (\%) $\downarrow$ & 7 & 0 & 0 & 3 & 0 & 0 \\
 & \cellcolor{black!8}Timed Out (\%) $\downarrow$ & \cellcolor{black!8}0 & \cellcolor{black!8}0 & \cellcolor{black!8}33 & \cellcolor{black!8}0 & \cellcolor{black!8}0 & \cellcolor{black!8}0 \\
 & \avgStepsHeader & 14.5$\pm$1.1 & 14.9$\pm$1.1 & 60.0$\pm$24.6 & 14.7$\pm$1.1 & 15.2$\pm$1.2 & 15.2$\pm$2.0 \\
\midrule
\multirow{4}{*}{\makecell[c]{Top\\Left}} & \cellcolor{black!8}Success (\%) $\uparrow$ & \cellcolor{black!8}\textbf{100} & \cellcolor{black!8}\underline{93} & \cellcolor{black!8}\textbf{100} & \cellcolor{black!8}\textbf{100} & \cellcolor{black!8}\textbf{100} & \cellcolor{black!8}\textbf{100} \\
 & \textcolor{black}{Fell in Pit} (\%) $\downarrow$ & 0 & 7 & 0 & 0 & 0 & 0 \\
 & \cellcolor{black!8}Timed Out (\%) $\downarrow$ & \cellcolor{black!8}0 & \cellcolor{black!8}0 & \cellcolor{black!8}0 & \cellcolor{black!8}0 & \cellcolor{black!8}0 & \cellcolor{black!8}0 \\
 & \avgStepsHeader & 12.6$\pm$1.3 & 12.8$\pm$1.4 & 13.5$\pm$1.2 & 12.5$\pm$1.1 & 12.7$\pm$0.9 & 12.5$\pm$1.4 \\
\midrule
\multirow{4}{*}{\makecell[c]{Top\\Center}} & \cellcolor{black!8}Success (\%) $\uparrow$ & \cellcolor{black!8}67 & \cellcolor{black!8}83 & \cellcolor{black!8}70 & \cellcolor{black!8}73 & \cellcolor{black!8}\textbf{100} & \cellcolor{black!8}\underline{93} \\
 & \textcolor{black}{Fell in Pit} (\%) $\downarrow$ & 33 & 17 & 30 & 27 & 0 & 7 \\
 & \cellcolor{black!8}Timed Out (\%) $\downarrow$ & \cellcolor{black!8}0 & \cellcolor{black!8}0 & \cellcolor{black!8}0 & \cellcolor{black!8}0 & \cellcolor{black!8}0 & \cellcolor{black!8}0 \\
 & \avgStepsHeader & 12.3$\pm$2.1 & 12.6$\pm$2.1 & 17.0$\pm$6.6 & 11.6$\pm$1.7 & 13.7$\pm$2.3 & 16.2$\pm$4.4 \\
\bottomrule
\end{tabular}%
}
\makebox[\textwidth][c]{%
\begin{minipage}{\wd\marbleplanningtablebox}
\captionsetup{skip=2pt,width=\linewidth,font=footnotesize,justification=raggedright,singlelinecheck=false}
\caption{Planning results across six marble maze sections (30 runs each).}
\label{tab:marbleplanning}
\end{minipage}
}
\makebox[\textwidth][c]{\usebox{\marbleplanningtablebox}}
\par\vspace{3pt}
\makebox[\textwidth][c]{%
\begin{minipage}{\wd\marbleplanningtablebox}
\footnotesize
\raggedright
\looseness=-1
\textcolor{black}{Success $=$ reaching \goal without falling into a pit.} Avg. Steps reports mean$\pm$std steps to goal over successful episodes. Time-outs occur when an episode reaches step 100. \textbf{Bold} is best, \underline{underline} is second best.\par
\end{minipage}
}
\vspace{-6mm}
\end{table*}

%% file: pih_coverage_table.tex
\newsavebox{\pihCoverageTableBox}
\begin{table*}[!t]
\centering
\footnotesize
\setlength{\tabcolsep}{3pt}
\renewcommand{\arraystretch}{1.12}
\captionsetup{skip=2pt,width=\textwidth,font=footnotesize}
\caption{Peg-in-hole empirical coverage and C-space volume over 90{,}000 held-out one-step transitions.}
\label{tab:pih_coverage}
\sbox{\pihCoverageTableBox}{%
\begin{tabular}{l!{\vrule width 0.45pt}l!{\vrule width 0.45pt}ccccc}
\toprule
\multicolumn{1}{c!{\vrule width 0.45pt}}{\makecell[c]{Metric}}
& \makecell[c]{Interaction Regime}
& \resultHeader{\lucca}
& \resultHeader{\pcp}
& \resultHeader{\ablationNoStrataHeader{4}}
& \resultHeader{\captureAblationHeader{1}}
& \resultHeader{\capturePrimaryHeader{4}} \\
\midrule

\multirow{4}{*}{\makecell[c]{Empirical\\Coverage (\%) $\uparrow$}}
& Aggregate
& 90.3
& 89.9
& 91.1
& 90.9
& 90.4 \\

& \cellcolor{black!8}Free Space (3D)
& \cellcolor{black!8}83.8
& \cellcolor{black!8}83.1
& \cellcolor{black!8}84.9
& \cellcolor{black!8}91.4
& \cellcolor{black!8}91.0 \\

& Point Contact (2D)
& 96.6
& 96.7
& 97.3
& 91.0
& 90.1 \\

& \cellcolor{black!8}Edge Contact (1D)
& \cellcolor{black!8}93.5
& \cellcolor{black!8}93.1
& \cellcolor{black!8}94.5
& \cellcolor{black!8}90.3
& \cellcolor{black!8}90.3 \\

\midrule
\makecell[c]{Avg. C-space\\Volume (ratio) $\downarrow$}
& Aggregate
& 1.03
& 0.70
& 1.09
& 0.81
& 1.00 \\

\bottomrule
\end{tabular}
}%

\makebox[\textwidth][c]{\usebox{\pihCoverageTableBox}}

\par\vspace{3pt}

\makebox[\textwidth][c]{%
\begin{minipage}{\wd\pihCoverageTableBox}
\footnotesize
\raggedright
User-specified coverage is $90\%$. The reported interaction regimes are derived from the realized landing contact fingerprint $\phi(c_{t+1})$. C-space volume is reported as a ratio relative to \textbf{\method}.\par
\end{minipage}
}
\vspace{-6mm}
\end{table*}

%% file: pih_inference_rand_table.tex
  \newsavebox{\pegplanningtablebox}
  \begin{table}[!t]
  \vspace{-3mm}
  \centering
  \footnotesize
  \setlength{\tabcolsep}{3pt}
  \renewcommand{\arraystretch}{1.12}

  \sbox{\pegplanningtablebox}{%
  \begin{tabular}{l!{\vrule width 0.45pt}cccccc}
  \toprule
  Metric
  & \resultHeader{\particleNoCP}
  & \resultHeader{\pcp}
  & \resultHeader{\lucca}
  & \resultHeader{\ablationNoStrataHeader{4}}
  & \resultHeader{\captureAblationHeader{1}}
  & \resultHeader{\capturePrimaryHeader{4}} \\
  \midrule

  \cellcolor{black!8}Success (\%) $\uparrow$
  & \cellcolor{black!8}14
  & \cellcolor{black!8}\underline{48}
  & \cellcolor{black!8}\underline{48}
  & \cellcolor{black!8}\underline{48}
  & \cellcolor{black!8}20
  & \cellcolor{black!8}\textbf{78} \\

  \avgStepsHeader
  & 39.9$\pm$13.3
  & 38.4$\pm$13.9
  & 41.5$\pm$16.6
  & 43.1$\pm$15.9
  & 36.2$\pm$16.4
  & 28.7$\pm$8.9 \\

  \bottomrule
  \end{tabular}%
  }

  \makebox[\linewidth][c]{%
  \begin{minipage}{\wd\pegplanningtablebox}
  \captionsetup{
      skip=2pt,
      width=\linewidth,
      justification=raggedright,
      singlelinecheck=false
  }
  \caption{Peg-insertion planning over 50 feasible Halton initial configurations.}
  \label{tab:pih_inference_rand}
  \end{minipage}
  }

  \makebox[\linewidth][c]{\usebox{\pegplanningtablebox}}

  \par\vspace{3pt}

  \makebox[\linewidth][c]{%
  \begin{minipage}{\wd\pegplanningtablebox}
  \raggedright
  Success denotes insertion within 75 controller steps (5.0 simulated seconds).
  Avg. Steps reports mean$\pm$std over successful episodes only.
  \end{minipage}
  }

  \vspace{-3mm}
  \end{table}

%% file: marble_testcasesPerMap_table.tex
\newsavebox{\marbleCoverageByMapTableBox}
\begin{table*}[!b]
\centering
\footnotesize
\setlength{\tabcolsep}{3pt}
\renewcommand{\arraystretch}{1.0}
\captionsetup{skip=2pt,width=\textwidth,font=footnotesize}
\caption{Marble labyrinth empirical coverage and C-space volume by maze section.}
\label{tab:marblecoveragebymap}
\sbox{\marbleCoverageByMapTableBox}{%
\begin{tabular}{l!{\vrule width 0.45pt}l!{\vrule width 0.45pt}l!{\vrule width 0.45pt}ccccc}
\toprule
\makecell[c]{Map} & \makecell[c]{Metric} & \makecell[c]{Stratum} & \lucca & \pcp & \ablationNoStrataHeader{8} & \captureAblationHeader{1} & \capturePrimaryHeader{8} \\
\midrule
\multirow{5}{*}{\makecell[c]{Bottom\\Left}} & \multirow{4}{*}{\makecell[c]{Empirical\\Coverage (\%) $\uparrow$}} & All Strata & 90.7 & 90.0 & 90.0 & 90.1 & 90.0 \\
 &  & \cellcolor{black!8}Free Space (2D) & \cellcolor{black!8}83.3 & \cellcolor{black!8}75.9 & \cellcolor{black!8}82.8 & \cellcolor{black!8}90.2 & \cellcolor{black!8}90.2 \\
 &  & Edge (1D) & 94.6 & 97.7 & 94.0 & 90.0 & 89.5 \\
 &  & \cellcolor{black!8}Corner (0D) & \cellcolor{black!8}96.8 & \cellcolor{black!8}98.5 & \cellcolor{black!8}93.8 & \cellcolor{black!8}90.3 & \cellcolor{black!8}92.4 \\
\cline{2-8}\noalign{\vskip 1.5pt}
 & \makecell[c]{Volume (ratio) $\downarrow$} & All Strata & 1.52 & 0.71 & 0.80 & 0.97 & 1.00 \\
\midrule
\multirow{5}{*}{\makecell[c]{Bottom\\Right}} & \multirow{4}{*}{\makecell[c]{Empirical\\Coverage (\%) $\uparrow$}} & All Strata & 91.1 & 90.0 & 90.2 & 90.3 & 90.0 \\
 &  & \cellcolor{black!8}Free Space (2D) & \cellcolor{black!8}83.7 & \cellcolor{black!8}76.0 & \cellcolor{black!8}85.2 & \cellcolor{black!8}90.1 & \cellcolor{black!8}89.9 \\
 &  & Edge (1D) & 94.5 & 96.7 & 92.6 & 89.8 & 89.9 \\
 &  & \cellcolor{black!8}Corner (0D) & \cellcolor{black!8}96.1 & \cellcolor{black!8}98.7 & \cellcolor{black!8}93.4 & \cellcolor{black!8}95.1 & \cellcolor{black!8}91.5 \\
\cline{2-8}\noalign{\vskip 1.5pt}
 & \makecell[c]{Volume (ratio) $\downarrow$} & All Strata & 1.34 & 0.64 & 0.82 & 1.07 & 1.00 \\
\midrule
\multirow{5}{*}{\makecell[c]{Center}} & \multirow{4}{*}{\makecell[c]{Empirical\\Coverage (\%) $\uparrow$}} & All Strata & 90.9 & 90.1 & 90.0 & 89.9 & 90.2\\
 &  & \cellcolor{black!8}Free Space (2D) & 81.6 & 74.8 & 81.9 & 89.8 & 89.5 \\
 &  & Edge (1D) & 94.8 & 96.6 & 93.4 & 89.9 & 90.5 \\
 &  & \cellcolor{black!8}Corner (0D) & 96.4 & 99.9 & 95.4 & 90.3 & 91.1 \\
\cline{2-8}\noalign{\vskip 1.5pt}
 & \makecell[c]{Volume (ratio) $\downarrow$} & All Strata & 1.30 & 0.52 & 0.84 & 1.07 & 1.00  \\
\midrule
\multirow{5}{*}{\makecell[c]{Top\\Center}} & \multirow{4}{*}{\makecell[c]{Empirical\\Coverage (\%) $\uparrow$}} & All Strata & 91.1 & 89.9 & 90.4 & 90.2 & 90.5 \\
 &  & \cellcolor{black!8}Free Space (2D) & \cellcolor{black!8}85.4 & \cellcolor{black!8}80.5 & \cellcolor{black!8}85.4 & \cellcolor{black!8}90.2 & \cellcolor{black!8}90.4 \\
 &  & Edge (1D) & 94.0 & 94.2 & 92.7 & 90.2 & 90.6 \\
 &  & \cellcolor{black!8}Corner (0D) & \cellcolor{black!8}94.9 & \cellcolor{black!8}99.6 & \cellcolor{black!8}94.8 & \cellcolor{black!8}90.5 & \cellcolor{black!8}90.0 \\
\cline{2-8}\noalign{\vskip 1.5pt}
 & \makecell[c]{Volume (ratio) $\downarrow$} & All Strata & 1.52 & 0.82 & 0.86 & 1.07 & 1.00 \\
\midrule
\multirow{5}{*}{\makecell[c]{Top\\Left}} & \multirow{4}{*}{\makecell[c]{Empirical\\Coverage (\%) $\uparrow$}} & All Strata & 90.7 & 90.0 & 90.1 & 90.2 & 90.2 \\
 &  & \cellcolor{black!8}Free Space (2D) & \cellcolor{black!8}82.3 & \cellcolor{black!8}76.1 & \cellcolor{black!8}84.3 & \cellcolor{black!8}90.3 & \cellcolor{black!8}90.1 \\
 &  & Edge (1D) & 94.0 & 95.6 & 92.2 & 90.2 & 90.2 \\
 &  & \cellcolor{black!8}Corner (0D) & \cellcolor{black!8}97.5 & \cellcolor{black!8}99.9 & \cellcolor{black!8}96.1 & \cellcolor{black!8}90.7 & \cellcolor{black!8}89.9 \\
\cline{2-8}\noalign{\vskip 1.5pt}
 & \makecell[c]{Volume (ratio) $\downarrow$} & All Strata & 1.63 & 0.97 & 1.00 & 1.08 & 1.00 \\
\midrule
\multirow{5}{*}{\makecell[c]{Center\\Right}} & \multirow{4}{*}{\makecell[c]{Empirical\\Coverage (\%) $\uparrow$}} & All Strata & 91.1 & 90.0 & 90.1 & 90.1 & 90.1 \\
 &  & \cellcolor{black!8}Free Space (2D) & \cellcolor{black!8}83.3 & \cellcolor{black!8}74.6 & \cellcolor{black!8}82.3 & \cellcolor{black!8}90.2 & \cellcolor{black!8}90.2 \\
 &  & Edge (1D) & 94.5 & 97.6 & 93.6 & 89.2 & 90.3 \\
 &  & \cellcolor{black!8}Corner (0D) & \cellcolor{black!8}97.7 & \cellcolor{black!8}97.6 & \cellcolor{black!8}95.9 & \cellcolor{black!8}96.3 & \cellcolor{black!8}88.6 \\
\cline{2-8}\noalign{\vskip 1.5pt}
 & \makecell[c]{Volume (ratio) $\downarrow$} & All Strata & 1.79 & 0.80 & 0.93 & 1.04 & 1.00 \\
\bottomrule
\end{tabular}
}
\usebox{\marbleCoverageByMapTableBox}
\par\vspace{3pt}
\begin{minipage}{\wd\marbleCoverageByMapTableBox}
\footnotesize
\raggedright
User-specified coverage is $90\%$. C-space volume is reported as a ratio relative to \textbf{\method}.\par
Number of test cases per map: Center: 76,464; Center Right: 63,504; Bottom Left: 68,688; Bottom Right: 82,944; Top Left: 62,208; Top Center: 73,872.\par
\end{minipage}
\end{table*}